\documentclass[final]{colt-noname}

\newif\ifCOLT

\newcommand{\SectionAppendix}{\ifCOLT Appendix\else Section\fi}

\def\notes{0}

\usepackage{amsmath}
\usepackage{amssymb} 

\usepackage{algorithm}
    \usepackage{algpseudocode}
    \usepackage{etoolbox}
\usepackage[normalem]{ulem}
\usepackage{paralist}
\setdefaultenum{(1)}{}{}{}

\usepackage{fancyhdr}
\usepackage{graphicx}
\usepackage{wrapfig}
\usepackage{ifthen}
\usepackage{bbm}
\usepackage{comment}
\usepackage{verbatim}
\usepackage{multicol}

\usepackage[T1]{fontenc} 

\usepackage{thmtools, thm-restate}
\usepackage{mathtools}
\mathtoolsset{showonlyrefs}

\usepackage{times}

    \hypersetup{hidelinks}

\makeatletter
\newcommand*{\algrule}[1][\algorithmicindent]{%
  \makebox[#1][l]{%
    \hspace*{.2em}
  }
}

\newcount\ALG@printindent@tempcnta
\def\ALG@printindent{%
    \ifnum \theALG@nested>0
    \ifx\ALG@text\ALG@x@notext
    \else 
    \unskip
    \ALG@printindent@tempcnta=1
    \loop
    \algrule[\csname ALG@ind@\the\ALG@printindent@tempcnta\endcsname]%
    \advance \ALG@printindent@tempcnta 1
    \ifnum \ALG@printindent@tempcnta<\numexpr\theALG@nested+1\relax
    \repeat
    \fi
    \fi
}
\patchcmd{\ALG@doentity}{\noindent\hskip\ALG@tlm}{\ALG@printindent}{}{\errmessage{failed to patch}}
\patchcmd{\ALG@doentity}{\item[]\nointerlineskip}{}{}{} 
\makeatother

    \newtheorem{claim}[theorem]{Claim}
    
    \newtheorem{fact}[theorem]{Fact}
    \newtheorem{observation}[theorem]{Observation}
    \newtheorem{problem}[theorem]{Problem}

\newcommand{\majoredit}{\color{black}}
\newcommand{\editdone}{\color{black}}

\newcommand{\defeq}{\stackrel{{\mbox{\tiny def}}}{=}}
\newcommand{\eps}{\varepsilon}

\DeclareMathOperator{\tr}{\mathrm{tr}}

\newcommand{\cA}{\mathcal{A}}

\newcommand{\cD}{\mathcal{D}}

\newcommand{\cM}{\mathcal{M}}
\newcommand{\cN}{\mathcal{N}}

\newcommand{\cU}{\mathcal{U}}

\newcommand{\cX}{\mathcal{X}}
\newcommand{\cY}{\mathcal{Y}}

\newcommand{\bbE}{\mathbb{E}}

\newcommand{\bbI}{\mathbb{I}}

\newcommand{\bbN}{\mathbb{N}}

\newcommand{\bbR}{\mathbb{R}}

\newcommand{\N}{N}

\ifnum\notes=1
\newcommand{\mynote}[3]{\marginpar{\tiny  \color{#1} #2: #3  }}

\else 
\newcommand{\mynote}[3]{}

\fi

\newcommand{\paren}[1]{{\left( {#1} \right)}}

\newcommand{\abs}[1]{{\left| {#1} \right|}}
\newcommand{\braces}[1]{{\left\{ {#1} \right\} }}
\newcommand{\norm}[1]{\left\lVert {#1} \right\rVert}
\newcommand{\brackets}[1]{\left[ {#1} \right]}
\newcommand{\ceiling}[1]{\left\lceil {#1} \right\rceil}
\newcommand{\floor}[1]{\left\lfloor {#1} \right\rfloor}

\newcommand{\ip}[1]{{\left\langle {#1} \right\rangle}}
\newcommand{\indicator}[1]{\mathbbm{1}\{ {#1} \} }

\newcommand{\simiid}{\overset{\mathrm{iid}}{\sim}}

\DeclarePairedDelimiterX{\infdivx}[2]{(}{)}{%
  #1\;\delimsize\|\;#2%
}

\newcommand{\supp}[1]{\ensuremath{\mathrm{supp}\paren{{#1}}}}

\newcommand{\PASS}{\ensuremath{\mathtt{PASS}}}
\newcommand{\FAIL}{\ensuremath{\mathtt{FAIL}}}

\newcommand{\stablecovariance}{\mathtt{StableCovariance}}
\newcommand{\stablemean}{\mathtt{StableMean}}
\newcommand{\score}{\mathtt{SCORE}}

\newcommand{\PTR}{\ensuremath{\cM_{\mathrm{PTR}}^{\eps,\delta}}}
\newcommand{\mainalg}{\cA_{\mathrm{main}}^{\eps,\delta,\lambda_0}}
\newcommand{\pass}{\ensuremath{\mathtt{PASS}}}
\newcommand{\fail}{\ensuremath{\mathtt{FAIL}}}

\newcommand{\m}{m}
\newcommand{\R}{R}
\newcommand{\Rsize}{M}
\newcommand{\xx}{x}
\newcommand{\yy}{y}

\newcommand{\covone}{\Sigma_1}
\newcommand{\covtwo}{\Sigma_2}

\newcommand{\inputthreshold}{\ensuremath{20 k \lambda_0}}
\newcommand{\ptrthreshold}{\ceiling{\frac{6\log 6/\delta}{\eps}} + 4}
\newcommand{\squarednoisescale}{\frac{1800}{(1-\gamma)^{2k+1}(1-k/n)^4}\cdot \frac{\lambda_{0} \log 12/\delta}{\eps^2 n^2}}

\newcommand{\smini}{S_{\ell}\setminus\braces{i^*}}

\newcommand{\psdjump}{\ensuremath{\gamma}}
\newcommand{\psdjumpdef}{\ensuremath{\frac{50}{9}\cdot \frac{\lambda_{0}}{n}}}

\newcommand{\covariancesclose}{
    Fix data set size $n$, outlier threshold $\lambda_0>0$, and discretization parameter $k\in \bbN$. 
    Assume $n \ge 20 k\lambda_{0}$.
    Let $\xx,\xx'$ be adjacent data sets and set
    \begin{align}
        \covone, \score &\gets \mathtt{StableCovariance}(\xx,\lambda_0,k) \\
        \covtwo, \score' &\gets \mathtt{StableCovariance}(\xx',\lambda_0,k).
    \end{align}
    Assume $\score, \score' < k$. 
    Let $\psdjump = \frac{50}{9}\cdot\frac{\lambda_{0}}{n}$.
    Then $\covone,\covtwo\succ 0$ and $(1-\psdjump)\covone \preceq \covtwo \preceq \frac{1}{1-\psdjump} \covone$.
    Furthermore, 
    \begin{align}
        \norm{\covone^{-1/2}\covtwo\covone^{-1/2} - \bbI}_{\tr}, \norm{\covtwo^{-1/2}\covone\covtwo^{-1/2} - \bbI}_{\tr} &\le \frac{25}{3}  \cdot \frac{1}{1-\psdjump}\cdot \frac{\lambda_0}{n} = O\left(\frac{\lambda_0}{n}\right).
    \end{align}
}

\newcommand{\meansclose}{
    Fix data set size $n$, dimension $d$, outlier threshold $\lambda_0\ge 1$ and discretization parameter $k\in\bbN$.
    Use reference set $\R\subseteq [n]$ with $\abs{\R}>6k$.
    Let $\xx$ and $\xx'$ be adjacent $d$-dimensional data sets of size $n$ that differ in index $i^*$.
    Assume that $n \ge 20 k\lambda_{0}$.
    Let $\covone, \covtwo\in\bbR^{d\times d}$ be positive definite matrices satisfying $(1-\psdjump)\covone \preceq \covtwo \preceq \frac{1}{1-\psdjump}\covone$ for $\psdjump = \psdjumpdef$.
    Let 
    \begin{align}
        \hat\mu, \score &\gets \stablemean(\xx, \covone, \lambda_0, k, \R) \\
        \hat\mu', \score' &\gets \stablemean(\xx', \covtwo, \lambda_0, k, \R).
    \end{align}
    If $\score, \score' < k$ and $\R$ is degree-representative for both $\xx$ and $\xx'$ (see Definition~\ref{def:degree_representative}), then $\norm{\hat\mu - \hat\mu'}_{\covone}^2 \le \frac{25}{(1-\psdjump)^{2k+1}(1-k/n)^4}\cdot  \frac{\lambda_{0}}{n^2}$. %
}

\title[Fast Private Mean and Covariance Estimation]{Fast, Sample-Efficient, Affine-Invariant 
    Private Mean and \\ Covariance Estimation for Subgaussian Distributions
}
\coltauthor{\Name{Gavin Brown\nametag{\thanks{Department of Computer Science, Boston University. Supported by NSF Awards CNS-2120667 and CCF-1763786 and an Apple Faculty Award.} }}
    \Email{grbrown@bu.edu} \\
    \Name{Samuel B. Hopkins\nametag{\thanks{Department of Electrical Engineering and Computer Science, Massachusetts Institute of Technology. Supported in part by funds from the MLA@CSAIL initiative and by NSF Award CCF-2238080.} }}
    \Email{samhop@mit.edu} \\
    \Name{Adam Smith\nametag{\footnotemark[1]}}
    \Email{ads22@bu.edu}
}

\begin{document}

\maketitle

\footnotetext{Code available at \url{https://github.com/gavinrbrown1/private-estimation-BHS23}}

\begin{abstract}
We present a fast, differentially private algorithm for high-dimensional \emph{covariance-aware} mean estimation with nearly optimal sample complexity. 
Only exponential-time estimators were previously known to achieve this guarantee.
Given $n$ samples from a (sub-)Gaussian distribution 
with unknown mean $\mu$ and covariance $\Sigma$, our $(\eps,\delta)$-differentially private estimator produces  
$\tilde{\mu}$ such that $\|\mu - \tilde{\mu}\|_{\Sigma} \leq \alpha$
as long as $n \gtrsim \tfrac d {\alpha^2} + \tfrac{d \sqrt{\log 1/\delta}}{\alpha \eps}+\frac{d\log 1/\delta}{\eps}$. 
The Mahalanobis error metric $\|\mu - \hat{\mu}\|_{\Sigma}$ measures the distance between $\hat \mu$ and $\mu$ relative to $\Sigma$; it characterizes the error of the sample mean.
Our algorithm runs in time $\tilde{O}(nd^{\omega - 1} + nd/\eps)$, where $\omega < 2.38$ is the matrix multiplication exponent.

We  adapt an exponential-time approach of %
\citet*{brown2021covariance}, 
giving efficient variants of stable mean and covariance estimation subroutines that also improve the sample complexity to the nearly optimal bound above.

Our stable covariance estimator can be turned to private covariance estimation for unrestricted subgaussian distributions.
With $n\gtrsim d^{3/2}$ samples, our estimate is accurate in spectral norm.
This is the first such algorithm using $n= o(d^2)$ samples, answering an open question posed by~\cite{alabi2022privately}.
With $n\gtrsim d^2$ samples, our estimate is accurate in Frobenius norm.
This leads to a fast, nearly optimal algorithm for private learning of unrestricted Gaussian distributions in TV distance.

\citet*{duchi2023fast} obtained similar results independently and concurrently.

\end{abstract}

\vspace{0.5in}
\setcounter{tocdepth}{2}
\setlength{\columnsep}{1cm}
\begin{multicols}{2}
{\tiny \tableofcontents}
\end{multicols}
\newpage

\section{Introduction}

We consider the statistical task of estimating the mean of a high-dimensional subgaussian distribution from independent samples under the constraint of differential privacy.
Differential privacy allows algorithm designers to control and reason about privacy loss in statistics and machine learning and is the standard approach for protecting the privacy of personal data, with adoption in academia, industry, and the United States government.
We focus on the following \emph{covariance-aware} version of the mean estimation problem.

\begin{problem}[Covariance-Aware Mean Estimation]
\label{prob:cov-aware}
Given $n$ samples from a subgaussian distribution $D$ on $\bbR^d$ with unknown mean $\mu$ and covariance $\Sigma$, output $\hat{\mu}$ such that $\|\hat{\mu} - \mu\|_\Sigma \leq \alpha$, where $\|\hat{\mu} - \mu\|_\Sigma$ denotes $\|\Sigma^{-1/2}(\hat{\mu} - \mu)\|_2$. %
\end{problem}

The \emph{Mahalanobis distance} $\|\hat{\mu} - \mu\|_\Sigma$ is the natural affine-invariant way to measure the statistical accuracy of an estimator for the mean.
An estimator with small Mahalanobis error automatically adapts to directions of low uncertainty in the location of $\mu$---in any particular direction, the estimator's error is scaled to the variance in that direction.  
The family of subgaussian distributions generalizes the Gaussian distribution; we defer definitions for now.

Absent privacy constraints,  the empirical mean already achieves the best possible guarantees for Problem~\ref{prob:cov-aware}, requiring $n \gtrsim d/\alpha^2$ samples. 
(Here $\gtrsim$ hides constants.) 
However, it does not generally protect its input data---for example, an adversary with a rough idea of the distribution could easily tell from observing the empirical mean if the data set contained an extreme outlier in any particular direction. We thus aim for estimators that are differentially private: 

\begin{definition}[Differential Privacy (DP)]
  Let $\cX$ and $\cY$ be sets.
  A (randomized) algorithm $A \, : \, \cX^n \rightarrow \cY$ satisfies $(\eps,\delta)$-DP if for  every $x = (x_1,\ldots,x_n) \in \cX^n$ and $x' = (x_1',\ldots,x_n') \in \cX^n$ such that $x,x'$ agree on all but one coordinate, $A(x)$ and $A(x')$ are $(\eps,\delta)$-indistinguishable (denoted $A(x) \approx_{(\eps,\delta)} A(x')$); that is, for any event $Y \subseteq \cY$,\
  \[
  \Pr[A(x) \in Y] \leq e^\eps \Pr[A(x') \in Y] + \delta \, .
  \]
\end{definition}

A differentially private estimator provides a strong guarantee: no matter what an outside adversary knows about the data set ahead of time, they will learn roughly the same things about any particular individual Alice \textit{regardless of whether Alice's data is used in the computation}~(as in, e.g., \cite{KasiviswanathanS14}). 
We call data sets that differ in one entry \emph{adjacent}.

Until recently, covariance-aware mean estimation subject to differential privacy was assumed by many researchers to require at least as many samples as privately estimating the \emph{covariance}---while the latter task also takes $O(d)$ samples nonprivately, $\Omega(d^{3/2})$ are required for private estimators (\cite{kamath2022new}).

\citet*{brown2021covariance} (henceforth, ``BGSUZ'')
disproved this assumption %
by giving
a modification of the exponential mechanism that solves Problem~\ref{prob:cov-aware} (for Gaussian data) as long as
\begin{equation}
    n \gtrsim \underbrace{\frac d {\alpha^2}}_{\substack{\text{nonprivate} \\ \text{sample} \\ \text{complexity}}} + \frac d {\alpha \eps} + \frac{\log 1/\delta}{\eps}
    \label{eq:sample-comp}
\end{equation}
This expression hides an absolute constant and an additive term of $\log(1/\eps \alpha)/(\eps \alpha)$.
The sample complexity is nearly tight, even when the covariance is known exactly: the first term in Equation~\eqref{eq:sample-comp} is necessary for nonprivate estimation, the second matches the $\Omega(d / \alpha \eps)$ lower bound of \cite{kamath_KLSU19,kamath2022new}, and the third is required even for estimating the mean of a one-dimensional Gaussian with unit variance (but unrestricted mean).
Subsequent work by~\citet*{liu2022differential} extended the exponential-mechanism-based approach to many statistical tasks; in particular, they solve Problem~\ref{prob:cov-aware} for general subgaussian distributions with the same near-optimal sample complexity.
Unfortunately, all these estimators appear to require at least exponential time to compute. 

Our main result is a polynomial-time algorithm for Problem~\ref{prob:cov-aware} with sample complexity also depending linearly on $d$: 

\begin{theorem}[Informal, see Theorem~\ref{thm:main}]
\label{thm:main-informal}
    Algorithm~\ref{alg:main} is $(\eps,\delta)$-differentially private. 
    Given 
    $$n \gtrsim \frac{d}{\alpha^2} + \frac{d\sqrt{\log 1/\delta}}{\alpha \eps} + \frac{d\log 1/\delta}{\eps}$$
    samples from a subgaussian distribution on $\bbR^d$ with mean $\mu$ and covariance $\Sigma$, with high probability it outputs $\tilde\mu$ such that $\norm{\tilde\mu - \mu}_{\Sigma}\le \alpha$.
    It runs in time $\tilde{O}\paren{nd^{\omega-1} + nd/\eps}$, where $\omega < 2.38$ is the matrix multiplication exponent.
\end{theorem}

When $d$ is large relative to $1/\eps$, the running time of our algorithm is dominated by the time to compute the covariance matrix,  $nd^{\omega -1}$. 
The terms $d \sqrt{\log 1/\delta} / (\alpha \eps)$ and $d \log(1/\delta)/ \eps$ in the sample complexity of our algorithm are slightly suboptimal; %
however, the nonprivate sample complexity continues to dominate for modest privacy parameters.

Our main algorithm relies on a new nonprivate covariance estimator $\stablecovariance$.
This estimator, combined with the ``Gaussian sampling mechanism'' of~\cite{alabi2022privately}, yields a fast algorithm for private covariance estimation that provides strong error guarantees in both spectral and Frobenius norm.
It is the first private algorithm achieving low spectral-norm error for unrestricted subgaussian distributions with $n=o(d^2)$ samples.
It also recovers guarantees similar to those of \citet*{ashtiani2021private} for estimating in Frobenius norm using $O(d^2)$ samples; such an error guarantee, in combination with our mean estimation guarantee, is known to suffice for learning unrestricted Gaussians distributions to low total variation distance.

\begin{theorem}[Informal, see Theorem~\ref{thm:private_covariance_estimation}]
    Algorithm~\ref{alg:privately_learn_Gaussians} is $(\eps,\delta)$-differentially private 
    and returns a vector $\tilde\mu$ and matrix $\tilde\Sigma$.
    Suppose it receives $n$ samples drawn i.i.d.\ from a subgaussian distribution with mean $\mu$ and covariance $\Sigma$
     over $\bbR^d$.
    \begin{compactitem}
    \item If $   
            n\gtrsim \frac{d}{\alpha^2} + \frac{d^{3/2} \sqrt{\log 1/\delta}}{\alpha \eps} + \frac{d \log 1/\delta}{\eps},
        $
        then, with high probability, $\norm{\Sigma^{-1/2}\tilde\Sigma \Sigma^{-1/2} - \bbI}_2\le \alpha$.
    \item If $  n\gtrsim \frac{d^2}{\alpha^2} + \frac{d^{2} \sqrt{\log 1/\delta}}{\alpha \eps} + \frac{d \log 1/\delta}{\eps},$ 
        then, with high probability, $\norm{\Sigma^{-1/2}\tilde\Sigma \Sigma^{-1/2} - \bbI}_F\le \alpha$;
        furthermore, if the  distribution is Gaussian, then  
         $\mathrm{TV}(\cN(\mu,\Sigma),\cN(\tilde\mu,\tilde\Sigma))=O(\alpha)$. 
    \end{compactitem}
    It runs in time $\tilde O(nd^{\omega-1} + nd/\eps)$, where $\omega<2.38$ is the matrix multiplication exponent.
\end{theorem}

Throughout, our accuracy analyses assume the underlying distribution has a full-rank covariance matrix.
There exist sample- and time-efficient preprocessing algorithms to deal with distributions where this assumption fails (\cite{singhal2021privately, kamath2021private, ashtiani2021private}).
Our privacy analysis makes no assumptions on the data.

In the remainder of this introduction we sketch our techniques and overview related work.

\subsection{Our Techniques}\label{sec:techniques}

\paragraph{Background: The Empirical Covariance Approach}%
In addition to their modification of the exponential mechanism,
BGSUZ give a second mean estimation algorithm, also requiring exponential time, which works for subgaussian distributions but has a worse dependence on the privacy parameters. This second algorithm is the basis for our efficient construction.

The basic idea of this algorithm is to first solve the problem for ``good'' data sets $y$, where goodness means a collection of conditions that are typical for subgaussian data. For example, a good data set $y$  should have no significant outliers as measured by $\norm{ (\cdot ) - \mu_{y}}_{\Sigma_y}$, where $\mu_{y}$ and $\Sigma_y$ are the empirical mean and covariance matrix of $y$.
BGSUZ show that, if we can restrict our inputs to good data sets, the mechanism that releases a single draw from $\cN(\mu_y, \sigma^2 \Sigma_y)$, for $\sigma\approx \frac{\sqrt{d}}{\eps n}$, 
is in fact differentially private. 
The data-dependent choice of covariance $\Sigma_y$ is crucial to achieving the Mahalanobis-distance guarantee, because it ensures 
that the noise added for privacy 
remains small in directions where $y$ itself has small variance.

Given a data set $x$, BGSUZ's algorithm first 
 projects $x$ to the nearest ``good'' data set $y$, then checks that $y$ is not too far from $x$ and finally, if that test passes, releases $\tilde{\mu}\sim\cN(\mu_y, \sigma^2 \Sigma_y)$.
The projection step ensures that privacy holds for all data sets. 

Although the noise addition is computationally efficient, the BGSUZ algorithm requires brute-force search over data sets for the projection step. It also comes with an additional cost in sample complexity---it roughly requires 
$n \gtrsim \frac{d}{\alpha^2} + \frac{d}{\alpha \eps^2}$. The extra factor of $\frac 1 \eps$ compared to Equation~\eqref{eq:sample-comp} comes from the fact that the projection is not \textit{stable}: it is possible for two neighboring data sets $x$ and $x'$ to be projected to ``good'' data sets $y$ and $y'$ that are far apart.

\paragraph{Overview of Our Approach}
In this work, we overcome the key computational bottleneck of  the empirical covariance approach;
along the way, we recover the optimal dependency on $\eps$. %

Our main technical tools are new efficient, \textit{non}-private estimators $\stablecovariance(x)$ and $\stablemean(x)$. 
We also design an algorithm $\score(x)$ which approximates the Hamming distance from $x$ to the nearest ``good'' dataset. 
We show that $|\score(x) - \score(x')| \leq 2$ for adjacent $x,x'$ and that if $\score(x)$ is small---as it is with high probability for subgaussian $x$---then $\stablecovariance(x)\approx\stablecovariance(x')$ and $\stablemean(x) \approx \stablemean(x')$.

At a  high level, on input $x$, our algorithm first privately tests if $\score(x)$ less than a threshold, roughly $1/\eps$.
If the test fails, we stop and return nothing (as in the ``PTR'' framework of \cite{DworkL09}). If the test passes, we output a sample from $\cN(\stablemean(x),~\sigma^2 \cdot \stablecovariance(x))$, where $\sigma\approx \frac{\sqrt{d}}{\eps n}$ is the scale parameter as before.

$\stablemean(x)$ and $\stablecovariance(x)$ stand in for the empirical parameters $\mu_y$ and $ \Sigma_y$ of the projection $y$ from BGSUZ's algorithm.
The advantages are computational tractability and stability: $\stablemean$ and $ \stablecovariance$ are more stable than their counterparts in BGSUZ, leading to our optimal dependence on $\eps$.
We establish the following key properties of $\stablecovariance$ (the ideas for $\stablemean$ are similar):
\begin{itemize}   
    \item \emph{Accuracy:} When the data come from a subgaussian distribution, with high probability we have $\stablecovariance(x) = \Sigma_x$, where $\Sigma_x$ is the empirical covariance of $x$.\footnote{Actually, the ``paired'' empirical covariance $\tfrac 2 n \sum_{i \leq n/2} (x_i - x_{i+n/2})(x_i - x_{i + n/2})^T$.}
    \item \label{property:stability}\emph{Stability:} If $x$ is a data set with $\score(x)\lesssim 1/\eps$, then for any adjacent $x'$,
    the matrices $\covone = \stablecovariance(x)$ and $\covtwo = \stablecovariance(x')$ are close. Specifically, $$\norm{ \covone^{-1/2}\covtwo\covone^{-1/2} - \bbI_{d}}_{\tr} \, , \, \norm{\covtwo^{-1/2}\covone\covtwo^{-1/2} - \bbI_{d} }_{\tr} \leq O(d/n), $$
    where $\|\cdot \|_{\tr}$ denotes the sum of absolute values of eigenvalues.
    \item \emph{Efficiency:} $\stablecovariance$ runs in polynomial time.
\end{itemize}

\paragraph{Stable Covariance}
We now sketch $\stablecovariance$.
Assume for now we have a data set $x$ that we believe to be drawn from $\cN(0,\Sigma)$ for some unknown $\Sigma$.
(We reduce to zero-mean case via a standard sample-pairing trick, see Equation~\eqref{def:empirical_mean_paired_covariance}.)
Ultimately, $\stablecovariance(x)$ produces a weight vector $w \in[0,1/n]^n$ and outputs $\Sigma_w=\sum_{i\in[n]} w_i \cdot x_i x_i^T$.
We will show the weights satisfy the following conditions.
\begin{enumerate}[(a)]
    \item \emph{Uniform on Subgaussian Inputs:} When the $x_i$ are drawn i.i.d. from a subgaussian distribution, we have $w = (1/n,\ldots,1/n)$ with high probability. \label{item:uniform}
    \item \emph{$\lambda$-Good Weighting:} For any $i \in [n]$ such that $w_i > 0$, we have $\norm{\Sigma_w^{-1/2} x_i}_2^2\le \lambda$. \label{item:good-weighting}
    \item \emph{Stability:} If $x$ is a data set with $\score(x)\lesssim 1/\eps$, then for any adjacent $x'$,
    the associated weight vectors $w,w'$ have $\|w - w'\|_1 \leq O(1/n)$. \label{item:stable-weighting}
\end{enumerate}

Here $\lambda>0$ is a hyperparameter, which we will eventually set to roughly $d$. 
Achieving Property (\ref{item:uniform}) implies our accuracy claim about $\stablecovariance$, since setting all weights to $1/n$ yields exactly the second moment matrix of $x$.
Since we have assumed $x$ is zero-mean, this serves in place of the empirical covariance.

Properties (\ref{item:good-weighting}) and (\ref{item:stable-weighting}) together imply the stability of $\stablecovariance$'s parameter estimates.
For intuition, consider the case where $w$ and $w'$ differ on exactly one value: suppose $w_i'=w_i+\eta_i$ for $\eta_i>0$.
Further suppose both weights are nonzero and that $x$ and $x'$ agree on index $i$.
Thus we have $\Sigma_{w'}=  \Sigma_{w} + \eta_i \cdot x_i x_i^T$.
\begin{align}
	\Sigma_{w}^{-1/2}\Sigma_{w'}\Sigma_{w}^{-1/2} - \bbI 
		&= \Sigma_{w}^{-1/2}\paren{\Sigma_{w}+ \eta_i \cdot x_ix_i^T} \Sigma_{w}^{-1/2} - \bbI \\
		&= \eta_i \cdot\Sigma_w^{-1/2} x_i x_i^T \Sigma_w^{-1/2}.
\end{align}
This is a rank-one matrix whose non-zero eigenvalue is exactly $\eta_i  \norm{x_i}_{\Sigma_w}^2$, which is at most $\eta_i \lambda$ by Property~\eqref{item:good-weighting}.

The remainder of the stability proof is relatively straightforward: in general $\Sigma_w-\Sigma_{w'}$ is the sum of rank-one matrices whose coefficients $\eta_i$ satisfy $\sum_i \abs{\eta_i} = \norm{w-w'}_1$.
The trace norm is the sum of the absolute values of the eigenvalues, so we arrive at an upper bound of roughly $\lambda\norm{w-w'}_1 \approx d/n$.
The proof requires some more work to account for the one index $i^*$ where $x_{i^*}\neq x_{i^*}'$, as well as accommodating indices where $w_i$ is zero and $w_i'$ is not (in this case $\norm{x_i}_{\Sigma_w}^2$ is not bounded by assumption).

\paragraph{Good Subsets}
We now discuss how $\stablecovariance$ computes the weight vector $w$.
\begin{definition}
A special case of a $\lambda$-good weighting is a $\lambda$-good \emph{subset}.
A subset $S \subseteq [n]$ is $\lambda$-good if its associated weight vector $w_S$ is $\lambda$-good, where $(w_S)_i = 1/n$ for $i \in S$ and $0$ otherwise. 
\end{definition}
$\stablecovariance$ computes a nested sequence of good subsets $S_0 \subseteq S_1 \subseteq \ldots \subseteq S_{2k}$, where $S_\ell$ is $\lambda_\ell$ good for some sequence of outlier thresholds $\lambda_1 \leq \ldots \leq \lambda_{2k}$ and $k$ is a parameter to set set later.
Then it averages a subset of the associated weight vectors to produce $w = \tfrac 1 k \sum_{\ell=k+1}^{2k} w_{S_\ell}$.
The subsets it chooses to average are described in the following lemma:

\begin{lemma}[Informal]
\label{lem:largest-intro}
  Let $x$ be a dataset and $\lambda > 0$ be an outlier threshold.
  There is a unique largest $\lambda$-good subset $S \subseteq [n]$ for $x$ which contains all other $\lambda$-good subsets.
  Furthermore, this largest good subset can be found using a greedy algorithm, $\mathtt{LargestGoodSubset}$.
\end{lemma}

$\stablecovariance$ takes $S_\ell$ to be the $\mathtt{LargestGoodSubset}$ for outlier threshold $\lambda_\ell$.
By Lemma~\ref{lem:largest-intro}, this implies immediately that $S_0 \subseteq \ldots \subseteq S_{2k}$, because $S_{\ell}$ is $\lambda_{\ell+1}$-good for all $\ell$.

For reasons we will see below, we choose the outlier thresholds according to $\lambda_{\ell + 1} = e^\eps \lambda_\ell$ and set $k \approx 1/\eps$, so that $\lambda_{2k} \leq O(\lambda_0)$.
We now have a nearly complete description of $\stablecovariance$, and enough information to see why the weights it computes satisfy properties~(\ref{item:uniform}) and~(\ref{item:good-weighting}) above.
We will ensure that if $[n]$ itself is a good subset, then $S_0 = \ldots = S_{2k} = [n]$, leading to Property~(\ref{item:uniform}) above.
Because the subsets $S_\ell$ are all $O(\lambda_{0})$-good, it will not be hard to show that $w$ is itself $O(\lambda_0)$-good, leading to Property~(\ref{item:good-weighting}) above.

Now we turn to Property~(\ref{item:stable-weighting}), which is contingent on $x,x'$ having $\score(x),\score(x') \lesssim 1/\eps$.
To proceed, we need to define $\score(x)$.

\paragraph{Score Function}
We have computed $S_\ell$, the largest $\lambda_\ell$-good subsets of $x$, for $\ell=0,1,\ldots,{2k}$.
While $\stablecovariance(x)$ averages over the items from $\ell=k+1,\ldots,2k$, $\score(x)$ uses the first half of the sequence:
\begin{align}
    \score(x) \defeq \min_{0\le \ell\le k} \paren{n - \abs{S_\ell} + \ell} \, .
        \label{eq:score_fn_defn}
\end{align}
As we will discuss below, this function has sensitivity 2: adjacent data sets have good subsets with similar sizes and outlier thresholds.
Our algorithm ultimately uses $\score'(x)=\min\braces{k, \score(x)}$, which simplifies some statements. %

For some intuition, note that if $[n]$ is good with respect to the outlier thresholds $\lambda_0,\ldots,\lambda_k$, then $\score(x) = 0$.
Together with the stability of $\score(x)$, this means $\score(x)$ roughly tracks the Hamming distance between $x$ and a dataset for which $[n]$ is a good subset.

\paragraph{Good Subsets on Adjacent Data Sets}
Towards establishing Property~(\ref{item:stable-weighting}) for the algorithm we have just described, we first consider what happens to a $\lambda$-good subset $S$ for a dataset $x$ if a single sample is \emph{removed} from $x$.

If we remove from $x$ a sample $x_i$ where $i \notin S$, then of course $S\setminus\braces{i}=S$, which is still $\lambda$-good for $x \setminus \{x_i \}$.
On the other hand, if we remove $x_i$ for some $i \in S$, the resulting $S\setminus\braces{i}$ may no longer be $\lambda$-good for $x \setminus \{x_i\}$; it may have many outliers!
See Figure~\ref{fig:bivariate_good_subset} for an illustration of this.
However, for the parameter regimes we consider, we can show that $S\setminus\braces{i}$ is still $\lambda'$-good for $x \setminus \{x_i\}$, where $\lambda' = e^{\eps}\lambda$.

This allows us to reason about pairs of data sets under removal \textit{and addition}, which is the notion we must consider to accommodate adjacent $x,x'$.
If $S$ is a $\lambda$-good subset for $x$ and data set $x'$ differs from it in index $i^*$, then $S\setminus\braces{i^*}$ is an $e^{\eps} \lambda$-good subset for $x'$.

\begin{figure}
    \centering
    \subfigure[Removing a point from a good subset][b]{
        \includegraphics[trim={37cm 11cm 5.5cm 6cm}, clip, width=0.4\textwidth]{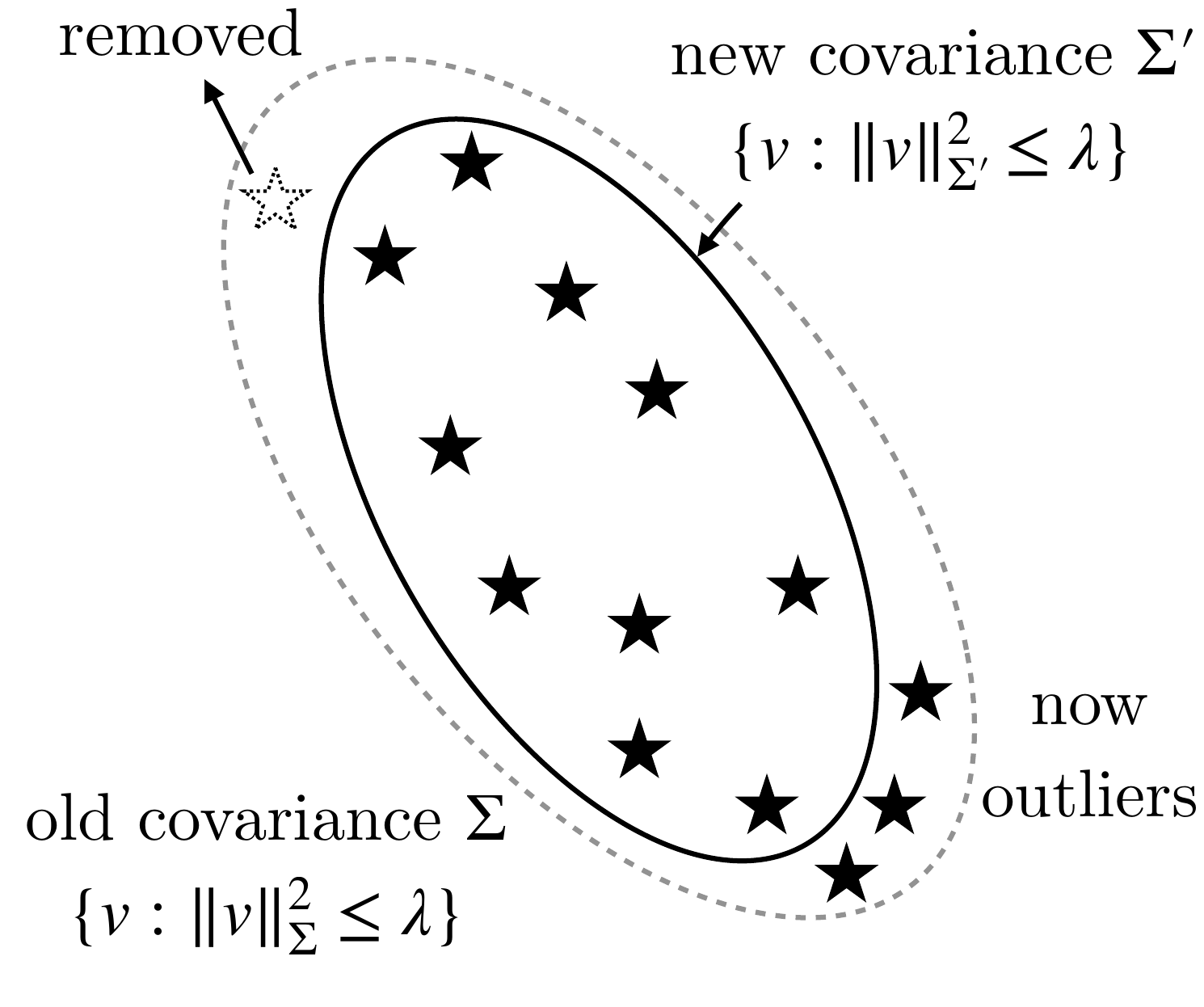}
        \label{fig:bivariate_good_subset}
    } \qquad
    \subfigure[Intertwined: $\tilde{S}_\ell \defeq S_\ell\setminus\braces{i^*}$][b]{
        \includegraphics[trim={7cm 11cm 36cm 5cm}, clip, width=0.4\textwidth]{figures/good_set_figures.pdf}
        \label{fig:good_sets_PH}
    }
    \caption{
        (\emph{a}) A good subset: no points are outliers with respect to the empirical covariance. 
        Removing one point may change the covariance (from $\Sigma$ to $\Sigma'$) and cause points to become (slight) outliers.
        (\emph{b}) Let $S_\ell$ denote the largest $\lambda_\ell$-good subset of $x$ and $T_\ell$ the same for adjacent $x'$.
        Let $\tilde{S}_{\ell}$ denote $S_\ell \setminus\braces{i^*}$, where $x,x'$ differ in index $i^*$ (and similarly $\tilde{T}_\ell$).
        These sets are intertwined: for all $\ell$, we have $\tilde{S}_{\ell}\cup \tilde{T}_{\ell} \subseteq \tilde{S}_{\ell+1}\cap \tilde{T}_{\ell+1}$.
    }
\end{figure}

\paragraph{Families of Largest Good Subsets on Adjacent Datasets}
Now let us return to the family $S_0,\ldots,S_{2k}$ of largest good subsets computed with respect to outlier thresholds $\lambda_0,\ldots,\lambda_{2k}$ on input $x$ and the analogous family $T_0,\ldots,T_{2k}$ for input $x'$ adjacent to $x$.
We prove a crucial ``intertwining'' relationship between $S_0,\ldots,S_{2k}$ and $T_0,\ldots,T_{2k}$; for a visual depiction, see Figure~\ref{fig:good_sets_PH}.

\begin{lemma}[Intertwining, Informal]
\label{lem:intertwining-intro}
Suppose $x,x'$ differ on index $i^*$.
For all $\ell$, defining $\tilde{S}_{\ell}\defeq S_\ell\setminus\braces{i^*}$ and $\tilde{T}_{\ell}\defeq T_\ell\setminus\braces{i^*}$, we know that $\tilde{S}_{\ell}\cup \tilde{T}_{\ell}$ is a subset of $\tilde{S}_{\ell+1}\cap \tilde{T}_{\ell+1}$.
\end{lemma}

\begin{proof}[\emph{sketch}]
From our discussion above, 
$S_\ell\setminus\braces{i^*}$ is a $\lambda_{\ell+1}$-good subset of both $x$ and $x'$ (assuming $\lambda_{\ell+1}\le e^{\eps} \lambda_\ell$, which will be true by construction).
Furthermore, $T_{\ell+1}$, the largest $\lambda_{\ell+1}$-good subset of $x'$, contains all $\lambda_{\ell+1}$-good subsets of $x'$, by Lemma~\ref{lem:largest-intro}, and similarly for $S_{\ell+1}$.
So, $T_{\ell+1},S_{\ell+1} \supseteq \tilde{S}_\ell$.
And since $\tilde{S}_\ell$ by definition does not contain $\braces{i^*}$, also $\tilde{T}_{\ell+1}, \tilde{S}_{\ell+1} \supseteq \tilde{S}_\ell$. A symmetric argument applies to $\tilde{T}_\ell$.
\end{proof}

Now we sketch the main ideas to show that intertwining implies property~(\ref{item:stable-weighting}), that $\|w-w'\|_1 \leq O(1/n)$, where
$w$ are the weights associated to input $x$ and $w'$ are those associated with an adjacent $x'$.
Since $\score(x), \score(x') < k\approx 1/\eps$, there exists $\ell < k$ such that $|S_\ell| \geq n - k$, and since $S_\ell \subseteq S_{k+1}$, also $|S_{k+1}| \geq n - k$.
By the nesting property $S_\ell \subseteq S_{\ell+1}$ and similarly for the $T_\ell$'s, we have
\[
\left |  \bigcap_{\ell = k+1}^{2k} S_\ell  \cap T_\ell \right | \geq n - O(k)
\]

By definition,
\begin{equation}
|w_i - w_i'| = \frac 1 {kn} \left | \sum_{\ell = k+1}^{2k} \indicator{i \in S_\ell} - \sum_{\ell = k+1}^{2k} \indicator{i \in T_\ell} \right | \, . \label{eq:diff_weight_intro}
\end{equation}
For $i \in \bigcap_{\ell = k+1}^{2k} S_\ell \cap T_\ell$, the difference in Equation~\eqref{eq:diff_weight_intro} is $0$.
This leaves just $O(k)$ indices $i$ where $|w_i - w'_i|$ could be nonzero.
We divide these into two cases. %

If $i = i^*$, the sums in Equation~\eqref{eq:diff_weight_intro} differ by at most $k$, so $|w_{i^*} - w'_{i^*}| \leq 1/n$.

For $i \neq i^*$, by the intertwining property (Lemma~\ref{lem:intertwining-intro}) we know that if $w_i\neq w_i'$ then there exists only one $\ell$ such that $ \indicator{i \in S_\ell} \neq \indicator{i \in T_\ell}$.
So for such $i$ we have $|w_i - w_i'| \leq 1/(kn)$, and the total contribution of those terms is $O(1/n)$.
This concludes our sketch of $\stablecovariance$.

\paragraph{From $\stablecovariance$ to Mean Estimation}
Let's return to the task of estimating the mean, given the output $\hat{\Sigma}$ of $\stablecovariance(x)$.
A natural strategy would be to draw $n$ fresh samples $x_i$ and set $\tilde{x}_i = \hat{\Sigma}^{-1/2} x_i$, then hand these samples to a \emph{non}-covariance-aware private mean estimation algorithm.
If $\hat{\Sigma}$ were actually \emph{private}, composition would show that this is differentially private.
However, our stability (rather than privacy) guarantees of $\stablecovariance$ do not translate into an overall privacy guarantee for this scheme.

Instead, the second phase of our algorithm computes a stable mean, to which we add noise.
Our estimator uses the same idea from $\stablecovariance$ of averaging over large subsets of the data which contain no outliers; it has a different notion of ``outlier'' that depends on pairwise distances between points (relative to the estimated covariance $\hat\Sigma$).
It is similar in flavor to recent approaches in private estimation that look for large sets of points (or ``cores'') which are all close to each other~\citep{tsfadia2022friendlycore,ashtiani2021private}.

\paragraph{Beyond Mean Estimation: Private Covariance Estimation and Learning of Gaussian Distributions}%
A key piece of our privacy proof is the stability of $\stablecovariance$: on adjacent inputs $x$ and $x'$ we either fail or produce covariance estimates $\covone\gets \stablecovariance(x)$ and $\covtwo\gets \stablecovariance(x')$ such that $\cN(0,\covone)\approx_{(\eps,\delta)} \cN(0,\covtwo)$.
Our main algorithm combines this with a stable mean estimation procedure, but by itself it gives rise to a differentially private algorithm: compute $\hat\Sigma\gets \stablecovariance(x)$ and, if this does not fail, return $z\sim\cN(0,\hat\Sigma)$.
This algorithm has dubious utility: $z$ contains no information about the mean of $x$ and little information about the covariance.
However, as \cite{alabi2022privately} recently observed, \emph{multiple} such samples are useful: 
given private samples $z_1,\ldots, z_{\N}\simiid \cN(0,\hat\Sigma)$, one can post-process them to produce a private covariance estimate $\tilde\Sigma = \frac 1 \N \sum_{i=1}^{\N} z_i z_i^T$. 
Our main analysis says that we can draw $\N = 1$ sample when $n\gtrsim d$; by advanced composition, we can draw $\N$ samples when $n \gtrsim d \sqrt{\N}$.
\cite{alabi2022privately} call this the \emph{Gaussian sampling mechanism}; they combine it with a stable estimator based on the sum-of-squares framework to privately and robustly learn the covariance of a Gaussian in polynomial time.
Combined with %
$\stablecovariance$, it yields a fast algorithm for private covariance estimation with strong error guarantees in both spectral and Frobenius norm, in particular achieving the optimal dimension-dependence for both.

\subsection{Comparison with a Concurrent Result}

Independently and concurrently, \citet{duchi2023fast} (henceforth, ``DHK'') provided an algorithm for private mean estimation with similar running time and accuracy guarantees.
Both papers fill in the basic framework of BGSUZ by devising stable mean and covariance estimators.
DHK's stable covariance estimator has a similar flavor to ours: iteratively removing outlier and updating the empirical covariance.
However, the analysis and algorithmic details are quite different.
Both DHK's stable mean estimator and our own consider pairwise distances between points (after rescaling by the stable covariance).
However, DHK's algorithm looks at distances within small, randomly generated subsets of the data.
Their algorithm and analysis, which make clever use of the randomness in the subset assignments, is substantially different from our own.

In the basic setting where the data are drawn i.i.d.\ from a full-rank subgaussian distribution, our accuracy guarantees are slightly stronger.
Informally, the noise we add for privacy incurs error of magnitude $\frac{d \sqrt{\log 1/\delta}}{\eps n}$, while the comparable term in their analysis is $\frac{d \log 1/\delta}{\eps n}$.
Additionally, their algorithm requires $n\gtrsim \frac{d \log^2 1/\delta}{\eps^2}$ examples to ensure privacy, while ours requires asymptotically fewer: $n\gtrsim \frac{d\log1/\delta}{\eps}$.

DHK extend their algorithm and analysis to address distributions that are heavy-tailed or rank-deficient.
Similar extensions would likely apply to our algorithm.
We apply our stable covariance estimator to the task of learning the covariance of a subgaussian distributions; DHK's analogous estimator would fill the same role with similar guarantees.

\subsection{Related Work}\label{sec:related_work}

In this section, we overview the most relevant work on learning the mean and/or covariance of high-dimensional, unrestricted Gaussian and subgaussian distributions subject to differential privacy.
These results deal with approximate differential privacy; algorithms satisfying stronger privacy notions (such as pure or concentrated DP) cannot be accurate without prior bounds on the parameters.

There is a great deal of research on differentially private statistics outside of these bounds.
The primer of~\cite{kamath2020primer} contains an introduction to differentially private statistics along with a survey.
One notable point is the work of \citet*{karwa2017finite}, which established approaches for learning univariate Gaussians with bounds that are independent of the data location and scale (using tools of \cite{DworkL09}).
Another influential work is that of~\citet*{kamath_KLSU19}, who gave algorithms for learning the mean and covariance of high-dimensional Gaussians that satisfy concentrated differential privacy. They  require prior bounds on the mean and covariance, but their error bounds depend on these quantities only polylogarithmically. 
Both papers proved lower bounds that apply to our setting.

The way we use the stability of our nonprivate parameter estimates is somewhat nontraditional.
Stability-like properties, such as bounds on Lipschitz constants or ``global sensitivity,'' come up frequently in the design of differentially private algorithms. Most often, the idea is to argue that some function is stable so that we can add limited noise to the output of the function (as with the Laplace and usual Gaussian mechanisms). By contrast, 
 we use the stability of $\stablecovariance$ to argue that the noise itself (added to $\stablemean(x)$) is indistinguishable on adjacent inputs.
This concept appears elsewhere recently in various forms \citep{kothari2021private, ashtiani2021private,tsfadia2022friendlycore,alabi2022privately}, though the details of our application are quite different.

\paragraph{Gaussian Mean Estimation}
The sample complexity of private covariance-aware mean estimation for unrestricted Gaussians is well-understood, with nearly matching upper \citep{brown2021covariance} and lower bounds \citep{karwa2017finite,kamath_KLSU19}.
Algorithms with nearly optimal sample complexity exist for subgaussian distributions as well \citep{liu2022differential},
with slightly stronger lower bounds known for some related families of distributions \citep{cai2019cost}.
Previously, known approaches for this task required exponential time or $\Omega(d^2)$ samples, corresponding to the best-known sample complexity for privately learning the covariance in spectral norm.
One can use such an approximation as a preconditioner: after rescaling, the distribution will be nearly isotropic and one can apply the techniques of \citet{karwa2017finite} dimension-by-dimension.

\paragraph{Gaussian Covariance Estimation} 
Differentially private Gaussian covariance estimation has received much attention recently, especially as the key stepping stone to privately learning a Gaussian to small total variation distance.
For this task the sample complexity is well-understood, with lower bounds \citep{vadhan2017complexity, kamath2022new} and clean information-theoretic upper bounds \citep{aden2021sample} matching the non-private sample complexity for modest privacy parameters.
A series of improving upper bounds has established polynomial-time algorithms nearly matching these results \citep{liu2022differential,kothari2021private,kamath2021private,ashtiani2021private,alabi2022privately,hopkins2022robustness}.
For more comprehensive overviews (which also discuss robustness and stronger notions of privacy) refer to \cite{alabi2022privately} and \cite{hopkins2022robustness}.

The task of privately producing a spectral-norm covariance approximation, often useful as a preconditioner, has seen less progress.
Nonprivately, this task requires only $n\approx d$ examples; our algorithm requires $d^{3/2}$.
This dependence on the dimension was recently proved be optimal in the regime of $\alpha = O(1/\sqrt{d})$ \citep{kamath2022new}. 
To the best of our knowledge, ours is the first differentially private polynomial-time algorithm achieving such a guarantee for unrestricted Gaussian distributions, closing an open question posed by~\cite{alabi2022privately}. %
Under the assumption that $\bbI \preceq \Sigma \preceq\kappa\bbI$, the standard Gaussian mechanism requires $\Omega(\mathrm{poly}(\kappa)  \cdot d^{3/2})$ samples and  the private preconditioner of~\cite{kamath_KLSU19} requires $\Omega(\mathrm{polylog}(\kappa)\cdot d^{3/2})$ samples (ignoring the dependence on other parameters).

\ifCOLT
\else
    \section{Main Result and Analysis}\label{sec:main}
\fi

 \newcommand{\accuracystatement}{
        For $\mu\in \bbR^{d}$ and positive definite $\Sigma\in\bbR^{d\times d}$, let $\xx = \xx_1,\ldots,\xx_n$ be drawn i.i.d.\ from a subgaussian distribution $\cD$ with parameter $K_{\cD}$, mean $\mu$, and covariance $\Sigma$.
        There exist absolute constants $K_1, K_2$, and $K_3$ such that, if $\lambda_0 = K_1 K_{\cD}^2 \paren{d + \log n/\beta}$ and 
        \begin{align}
            n\ge K_2 K_{\cD}^2 \cdot\frac{\log 1/\delta}{\eps} \paren{d+ \log\paren{ \frac{K_{\cD} \log 1/\delta}{\eps \beta}}},
        \end{align}
        then, with probability at least $1-\beta$, Algorithm~\ref{alg:main} returns $\hat{\mu}$ such that
        \begin{align}
            \norm{\hat\mu-\mu}_{\Sigma}\le K_3 K_{\cD}\paren{ \sqrt{\frac{d + \log 1/\beta}{n}} + \frac{d \sqrt{\log 1/\delta}}{\eps n}+ \frac{\log n/\beta \sqrt{\log 1/\delta}}{\eps n}}.
        \end{align}
    }

\begin{theorem}[Main Theorem]\label{thm:main}
    Fix $\eps,\beta \in (0,1)$, $\delta\in (0,\eps/10]$, and $n,d\in\bbN$.
    Algorithm~\ref{alg:main} takes a data set of $n$ points, each in $\bbR^d$, privacy parameters $\eps,\delta$, and an outlier threshold $\lambda_0$.
    \begin{itemize}
        \item For any $\lambda_0\ge 10$, 
            Algorithm~\ref{alg:main} is $(\eps, \delta )$-differentially private. 
        \item \accuracystatement
        
        \item Algorithm~\ref{alg:main} can be implemented to require: one product of the form $A^T A$ for $A\in\bbR^{n\times d}$, two products of the form $AB$ for $A\in \bbR^{n\times d}$ and $B\in \bbR^{d\times d}$, one inversion of a matrix in $\bbR^{d\times d}$ to logarithmic bit complexity, and further computational overhead of $\tilde{O}(nd/\eps)$.
    \end{itemize}
\end{theorem}

\ifCOLT
\else
    We prove Theorem~\ref{thm:main} by proving each subclaim separately: privacy in Lemma~\ref{lemma:privacy}, accuracy in Lemma~\ref{lemma:accuracy}, and running time in Lemma~\ref{lemma:running_time}.
\fi

\begin{algorithm}
    
    \SetAlgoLined
    \SetKwInOut{Input}{input}
    \SetKwInOut{Require}{require}
    
    \Input{data set $\xx\in \bbR^{n\times d}$; privacy parameters $\eps,\delta$; outlier threshold $\lambda_0$}
    \BlankLine
    \tcc{Initialize}
    $k\gets \ptrthreshold; \quad \Rsize \gets 6k + \lceil18 \log 16 n/\delta\rceil; \quad c^2 \gets \frac{810 \lambda_0 \log 12/\delta}{\eps^2 n^2}$\;
    
    $\R \sim \mathrm{Uniform}\paren{\braces{\R'\subseteq [n]: \abs{\R'}=\Rsize}}$\;

    \BlankLine
    \tcc{Check input size}
    \If{$n < 20 k \lambda_0$}{
        \KwRet \FAIL
    }

    \BlankLine
    \tcc{Compute nonprivate parameter estimates}
    $\hat\Sigma, \score_1 \gets \stablecovariance(\xx,\lambda_0,k)$\tcc*[r]{$\hat\Sigma\in\bbR^{d\times d}, \score_1\in \bbN$}
    $\hat\mu, \score_2 \gets \stablemean(\xx, \hat\Sigma, \lambda_0, k, \R)$\tcc*[r]{$\hat\mu\in\bbR^d, \score_2\in \bbN$}
    
    \BlankLine
    \tcc{Test the scores and release}
    \eIf{$\cM_{\mathrm{PTR}}^{\eps/3,\delta/6}\paren{\max\braces{\score_1,\score_2}} = \PASS$}{
        \KwRet{$\tilde\mu \sim \cN(\hat\mu, c^2\hat\Sigma)$}\;
    }{
        \KwRet \FAIL
    }

    \caption{Private Mean Estimation, $\mainalg(\xx)$}
    \label{alg:main}
\end{algorithm}

\ifCOLT
\else
    \subsection{Privacy Analysis}\label{sec:privacy_analysis}

Let $\xx$ and $\xx'$ be adjacent data sets.
Our main algorithm, Algorithm~\ref{alg:main}, either fails or computes nonprivate estimates $(\hat\mu, \hat\Sigma)$ and outputs a sample $\tilde\mu \sim \cN(\hat\mu, c^2\hat\Sigma)$ for some real number $c$.
By standard propose-test-release-style analysis, it suffices to establish that (i) the probabilities of failing under $\xx$ and $\xx'$ are $(\eps, \delta)$-indistinguishable and (ii) for any pairs $(\hat\mu,\hat\Sigma)$ computed by Algorithm~\ref{alg:main} on $x$ and $(\hat\mu', \hat\Sigma')$ computed on $x'$, if we do not fail then
\begin{align}
    \cN(\hat\mu, c^2 \hat\Sigma) \approx_{(\eps,\delta)} \cN(\hat\mu', c^2\hat\Sigma').
\end{align}
The indistinguishability of failure probabilities follows from the low sensitivity of our score function.
As in BGSUZ, to establish the indistinguishability of Gaussians we ``change the mean'' and ``change the covariance'' separately and use composition to argue that the result is indistinguishable.
We prove the overall privacy statement in Lemma~\ref{lemma:privacy} at the end of this subsection.

We call our propose-test-release function on the maximum of the two scores (one for the covariance and one for the mean).
We begin by showing that this maximum has low sensitivity.
\begin{lemma}\label{lemma:main_sensivity}
    Fix outlier threshold $\lambda_0>0$, reference set $\R\subseteq [n]$, privacy parameters $\eps>0$ and $0<\delta<1$, and discretization parameter $k\in \bbN$.
    Assume $n \ge 20 k \lambda_{0}$.
    Let $\xx$ and $\xx'$ be adjacent data sets of size $n$.
    Let 
    \begin{align}
        \hat\Sigma, \score_1 &\gets \stablecovariance(\xx,\lambda_0,k) \\
        \hat\Sigma', \score_1' &\gets \stablecovariance(\xx',\lambda_0,k) \\
        \hat\mu, \score_2 &\gets \stablemean(\xx,\hat\Sigma,\lambda_0,k,\R) \\
        \hat\mu', \score_2' &\gets \stablemean(\xx',\hat\Sigma',\lambda_0,k,\R).
    \end{align}
    Then $\abs{\max\braces{\score_1,\score_2} - \max\braces{\score_1',\score_2'}}\le 2$.
\end{lemma}
\begin{proof}
    Lemma~\ref{lemma:good_set_sensitivity}, proved in Section~\ref{sec:cov:sensitivity}, says that $\abs{\score_1 - \score_1'}\le 2$, 
    provided $n \ge 20\lambda_{0}$. 
    (Recall $\m$, the size of the paired data set, is $n/2$).
    This is satisfied by assumption, since $k> 1$.

    Lemma~\ref{lemma:core_sensitivity}, proved in Section~\ref{sec:mean:sensitivity}, says that $\abs{\score_2-\score_2'}\le 2$ provided
    $n\ge 20k\lambda_{0}$ and $(1-\psdjump)\hat\Sigma \preceq \hat\Sigma' \preceq \frac{1}{1-\psdjump} \hat\Sigma$ for $\psdjump=\frac{40}{7}\cdot\frac{\lambda_0}{n}$.
    The first holds by assumption.
    By Lemma~\ref{lemma:covariances_close} (below), if $\score_1,\score_1'<k$, the second condition also holds.
    If one of $\score_1$ or $\score_1'$ is equal to $k$, then we are still low-sensitivity:
    assuming without loss of generality that $\score_1=k$, we know that $\score_1'\ge k-2$, and thus 
    \begin{align}   
        \max\braces{\score_1, \score_2} &= k \\
        \max\braces{\score_1', \score_2'} &\ge \score_1' \ge k-2.
    \end{align}
    We have used the fact that all four scores are at most $k$.
\end{proof}

Algorithm~\ref{alg:main} feeds this maximum of the two scores to $\PTR$, a private propose-test-release-style check that the score is low and it is safe to proceed.
The textbook approach to this task uses Laplace noise; our mechanism is an analogue of this that always passes on zero inputs and always fails on large inputs.
This modification cleans up our arguments, decoupling parameter stability from the (random) outcome of the private check. 
We provide the (elementary) proof of Claim~\ref{claim:PTR} in Appendix~\ref{app:deferred}.
\newcommand{\PTRclaimstatement}{
    Fix $0<\eps\le 1$ and $0< \delta\le \frac{\eps}{10}$.
    There is an algorithm $\PTR: \bbR\to \braces{\pass,\fail}$ that satisfies the following conditions:
    \begin{enumerate}
        \item Let $\cU$ be a set and $g:\cU^n \to \bbR_{\ge 0}$ a function.
            If, for all $x,x'\in \cU^n$ that differ in one entry,  $\abs{g(x)-g(x')}\le 2$, then $\PTR(g(\cdot))$ is $(\eps,\delta)$-DP.
        \item $\PTR(0)=\pass$.
        \item For all $z\ge  \frac{2 \log 1/\delta}{\eps}+4$, $\PTR(z)=\fail$.
    \end{enumerate}
}
\begin{claim}\label{claim:PTR}
    \PTRclaimstatement
\end{claim}

We now present the stability guarantees for our nonprivate parameter estimates.
Recall that $\stablecovariance$ is deterministic and outputs an estimate $\hat\Sigma\in \bbR^{d\times d}$ along with an integer $\score$.
The value of $\score$ is 2-sensitive (under adjacent inputs) and, when $\score$ is not too large, the parameter estimate $\hat\Sigma$ is itself stable in the sense we need for privacy.

\begin{lemma}\label{lemma:covariances_close}
    \covariancesclose
\end{lemma}
The bulk of Section~\ref{sec:covariance} is devoted to proving Lemma~\ref{lemma:covariances_close}.
Together with the following claim, it immediately implies indistinguishabilty of zero-mean, rescaled Gaussians.
\begin{claim}\label{claim:covariance_indistinguishability}
    Fix $\eps\in (0,1)$ and $\delta\in (0,1/6000]$ and let $\Sigma_1, \Sigma_2 \in \bbR^{d\times d}$ be positive definite matrices.
    If
    \begin{align}
        \norm{ \Sigma_1^{-1/2}\Sigma_2\Sigma_1^{-1/2}- \bbI }_{\tr}, \norm{ \Sigma_2^{-1/2}\Sigma_1\Sigma_2^{-1/2}- \bbI }_{\tr} \le \frac 2 3 \cdot \frac{\eps}{\log 2/\delta},
    \end{align}
    then $\cN(0, \Sigma_1)\approx_{(\eps,\delta)} \cN(0, \Sigma_2)$.
\end{claim}
Claim~\ref{claim:covariance_indistinguishability} follows from a straightforward Gaussian concentration argument.
See BGSUZ, Lemma 4.15, and note that their proof establishes the stronger (by constant factors) statement we present.
See also \cite{alabi2022privately}, whose similar Theorem 5.1 we use in Section~\ref{sec:private_covariance_estimation}.

The main result in Section~\ref{sec:mean} is similar: $\stablemean$ is deterministic and returns an estimate $\hat\mu\in \bbR^d$ and an integer $\score$.
The latter is $2$-sensitive and, when it is not too big, the mean estimates are stable.
We have a few additional concerns: on inputs $\xx$ and $\xx'$, we might have previously calculated slightly different covariances $\covone$ and $\covtwo$. 
Thus we analyze the stability of $\stablemean$ when simultaneously moving from $(\xx,\covone)$  to $(\xx',\covtwo)$.
Furthermore, for computational efficiency $\stablemean$ accepts a set $\R\subseteq [n]$ of ``reference points'' on which to estimate whether points are outliers.
For our stability conditions to hold, this set needs to be both sufficiently large and sufficiently representative (in a precise sense: see Definition~\ref{def:degree_representative}). 

\begin{lemma}\label{lemma:means_close}
    \meansclose
\end{lemma}

Later we will simplify the expression in the conclusion to bound the sensitivity with $O(\lambda_0/n)$.
The nonprivate means are close in exactly the sense we need to establish indistinguishability of Gaussians (with the same covariance).
Recall that $\norm{\hat\mu - \hat\mu'}_{\covone}^2 = \lVert\covone^{-1/2}\hat\mu - \covone^{-1/2}\hat\mu'\rVert_{2}^2$.
\begin{claim}[Standard Gaussian Mechanism]\label{claim:gaussian_mechanism_privacy}
    Let vectors $u$ and $v$ satisfy $\norm{u-v}_2^2 \le \Delta^2$.
    For any $\eps,\delta\in(0,1)$, let $c^2 = \Delta^2 \cdot \frac{2 \log 2/\delta} {\eps^2}$.
    Then $\cN(u-v,c^2\bbI)\approx_{(\eps,\delta)}\cN(0,c^2\bbI)$.
\end{claim}

We are ready to prove, via standard tools, the privacy guarantee for Algorithm~\ref{alg:main}.
\begin{lemma}[Main Privacy Claim]\label{lemma:privacy}
    Fix $0<\eps<1$ and $0<\delta< \frac{\eps}{10}$.
    For any $\lambda_0 \ge 1$, Algorithm~\ref{alg:main} is $(\eps,\delta)$-differentially private.
\end{lemma}
\majoredit
\begin{remark}
    The proof below shows that $R$ is degree-representative with probability at least $1-\delta$ and then accounts for the privacy loss across propose-test-release, the difference in covariances, and the difference in means.
    We allocate $\eps$ evenly across these steps, but different allocations may achieve lower error (by constant factors).
    In particular, as $n$ grows, one may allocate more budget to the difference in means, thereby reducing the scale of the added noise.
\end{remark}
\editdone
\begin{proof}
    Set $\eps' = \eps/3$ and $\delta'= \delta/6$.
    We will prove that Algorithm~\ref{alg:main} is $(3\eps', 6\delta')$-differentially private.
    
    The algorithm requires $n\ge \inputthreshold$ to proceed.
    Assume this lower bound is true, since otherwise the algorithm immediately fails (and thus is private).
    Note that this, along with the bounds on $\eps$ and $\delta$, imply that $n>k\lambda_{2k}$.

    Fix adjacent data sets $\xx$ and $\xx'$.
    Since $\Rsize=\abs{\R} > 18\log 16n/\delta = 18\log 4 n/\delta'$, Claim~\ref{claim:hypergeometric_tail} and a union bound over all data points in $\xx$ and $\xx'$ establish that $\R$ is degree-representative for $\xx$ and $\xx'$ with probability at least $1-\delta'$.
    We will show that, conditioned on $\R$ being degree-representative for $Y$ and $Y'$, the mechanism is $(3\eps', 5\delta')$-differentially private. 
    This will prove that the entire algorithm is private with parameters $(3\eps',6\delta')=(\eps,\delta)$.

    We want $\cM_{\mathrm{PTR}}^{\eps',\delta'}$ to be $(\eps',\delta')$-differentially private.
    For any fixed reference set $\R=r\subseteq [n]$, Lemma~\ref{lemma:main_sensivity} and Claim~\ref{claim:PTR} tell us 
    this happens when $n\ge 20k\lambda_{0}$ and $k \ge \frac{2\log 1/\delta'}{\eps'} + 4 = \frac{6\log 6/\delta}{\eps} + 4$.
    We assumed that this is true.

    When we do not fail, by the guarantees of $\PTR$ we know that both scores are strictly less than $k$.
    Let $(\mu_1, \covone)$ be the nonprivate parameters computed on $\xx$ and $(\mu_2,\covtwo)$ the nonprivate parameters computed on $\xx'$.
    Lemma~\ref{lemma:covariances_close} and Claim~\ref{claim:covariance_indistinguishability} together imply that $\cN(0,\covone)\approx_{(\eps',\delta')} \cN(0,\covtwo)$.
    To apply Lemma~\ref{lemma:covariances_close}, we again require $n \ge 20k\lambda_{0}$.
    The lemma establishes a trace norm bound of $\Delta = \frac{25}{3} \cdot \frac{1}{1-\psdjump}\cdot \frac{\lambda_0}{n}$ for $\psdjump = \psdjumpdef$.
    We can upper bound $\psdjump$: by our assumptions on $n, \eps,$ and $\delta$, we have $k\ge 56$ and thus $\psdjump = \psdjumpdef \le \frac{5}{1008}\le \frac{1}{200}$. This implies $\Delta\le \frac{5000}{1791}\cdot \frac{\lambda_0}{n}$.
    To apply Claim~\ref{claim:covariance_indistinguishability}, we require $\Delta \le \frac 2 3 \cdot \frac{\eps'}{\log 2/\delta'}$. 
    (This claim also requires $\delta'\le 1/6000$, which is satisfied by assumption on $\delta$.)
    Rearranging, we see that $(\eps',\delta')$-indistinguishability of $\cN(0,\covone)$ and $\cN(0,\covtwo)$ requires 
    \majoredit
    \begin{equation}
        n \ge \frac{5000 \lambda_0}{1791}\cdot \frac{3 \log 2/\delta'}{2 \eps'}
            = \frac{2500}{597}\cdot \frac{\lambda_0 \log 12/\delta}{\eps}.
    \end{equation}
    This condition is implied by the assumptions that $n\ge 20 k \lambda_0$ and $\delta\le \frac{1}{1000}$.
    \editdone
    In particular, this implies $\cN(\mu_2, c^2\covone) \approx_{(\eps',\delta')} \cN(\mu_2,c^2\covtwo)$, as adding a fixed vector or multiplying by a fixed value do not affect indistinguishability.

    \majoredit
    Lemma~\ref{lemma:means_close}
    guarantees that $\norm{\mu_1 - \mu_2}_{\covone}^2 \le \frac{25}{(1-\gamma)^{2k+1}(1-k/n)^4}\cdot\frac{ \lambda_{0}}{n^2}$, for $\psdjump=\psdjumpdef$.
    We now bound the expressions in the denominator using the assumptions of $n\ge 20 k\lambda_0$, $\eps\le 1$, and $\delta\le 10^{-3}$.
    We calculated before that $\gamma \le \frac{1}{200}$.
    This in turn means $\frac{1}{1-\gamma}= 1 + \frac{\gamma}{1-\gamma} \le 1 + \frac{200 \gamma}{199}$. Applying $1+x\le e^x$ and our definitions, we have
    \begin{align}
        \paren{\frac{1}{1-\gamma}}^{2k+1} &\le \exp\left\{ \frac{200}{199}\cdot \frac{5}{18 k}\cdot(2k+1) \right\} \le \frac{65}{37}.
    \end{align}
    Next, since $\lambda_0 \ge 10$, we have $1-k/n\ge 1 - 1/20\lambda_0 \ge 199/200$, which allows us to bound $(1-k/n)^{-4}\le 1.021$, arriving at
        $\norm{\mu_1 - \mu_2}_{\covone}^2 \le \frac{45 \lambda_0}{n^2}$.
    \editdone
    
    With Claim~\ref{claim:gaussian_mechanism_privacy}, this means that $\cN(\mu_1, c^2\covone)\approx_{(\eps',\delta')}\cN(\mu_2, c^2\covone)$, where 
    \begin{align}
        c^2= \paren{\frac{45\lambda_{0}}{n^2}}\paren{\frac{2\log 2/\delta'}{\paren{\eps'}^2}}
            = 810 \cdot \frac{\lambda_0 \log 12/\delta}{\eps^2 n^2},
    \end{align}
    as in Algorithm~\ref{alg:main}.
    To apply Lemma~\ref{lemma:means_close}, we require a few conditions. 
    First, $\abs{\R}>6k$ and $n\ge 20k\lambda_{0}$ hold by construction.
    The condition that $(1-\psdjump)\covone\preceq \covtwo\preceq \frac{1}{1-\psdjump}\covone$ for $\psdjump=\psdjumpdef$ is satisfied; it is a consequence of Lemma~\ref{lemma:covariances_close}.
    Finally, the scores are strictly less than $k$ (since we did not fail) and $\R$ is degree-representative for $\xx$ and $\xx'$ (by assumption).
    
    We now apply Fact~\ref{fact:group_privacy} to combine the statements of $\cN(\mu_1, c^2\covone)\approx_{(\eps', \delta')} \cN(\mu_2,c^2\covone)$ and $\cN(\mu_2, c^2\covone)\approx_{(\eps', \delta')} \cN(\mu_2,c^2\covtwo)$.
    We get that $\cN(\mu_1, c^2\covone)\approx_{(2\eps', (1+e^{\eps'})\delta')} \cN(\mu_2,c^2\covtwo)$.
    Since $e^{\eps'}\le e\le 3$, we have the same statement with values $(2\eps',4\delta')$.
    Basic composition (Fact~\ref{fact:basic_composition}) allows us to combine the guarantees for these Gaussians with those for $\PTR$ to establish that the algorithm, conditioned on the fact that $\R$ is degree-representative for $x$ and $x'$, is $(3\eps', 5\delta')$-differentially private.
    Since $\R$ fails to be degree-representative for $x$ and $x'$ with probability at most $\delta'$, we have finished the proof.
\end{proof}

\subsection{Accuracy Analysis}\label{sec:accuracy}

Let $\xx$ be a $d$-dimensional data set of size $n$.
We define the empirical mean and paired empirical covariance:
\begin{align}
    \mu_{\xx} \defeq \frac{1}{n}\sum_{i=1}^n \xx_i
        \quad\text{and}\quad
    \Sigma_{\xx} \defeq \frac{1}{\sqrt{2} \floor{n/2}}  \sum_{i=1}^{\floor{n/2}} \paren{\xx_i - \xx_{i+\floor{n/2}}}\paren{\xx_i - \xx_{i+\floor{n/2}}}^T. 
        \label{def:empirical_mean_paired_covariance}
\end{align}
 This ``pairing'' is a standard trick that centers the data at the cost of halving our sample size: given input $\xx=(\xx_1,\ldots, \xx_n)$ we construct $\yy=(\yy_1,\ldots,\yy_m)$ by setting $\yy_i = (1/\sqrt{2})(\xx_i - \xx_{i+\m})$  where $\m=\floor{n/2}$. 
If we have $\xx_i, \xx_{i+m}\overset{\mathrm{iid}}{\sim}\cN(\mu,\Sigma)$, then $\yy_i\sim \cN(0,\Sigma)$.
The transformation preserves adjacency: if $\xx$ and $\xx'$ are adjacent, then so are $\yy$ and $\yy'$.

We now introduce a notion of good data sets, on which $\stablecovariance$ and $\stablemean$ ``let the data through,'' always returning exactly the paired empirical covariance and empirical mean, respectively.
Furthermore, Algorithm~\ref{alg:main} never returns $\fail$ on such inputs.
We formalize these points in Observation~\ref{obs:pass_on_good_data}, which is simple but serves as the linchpin of our accuracy analysis.

\begin{definition}[Well-Concentrated]\label{def:well_concentrated}
    Fix $\lambda_0>0$ and let $\xx$ be a $d$-dimensional data set of size $n$ samples.
    Let $\Sigma_{\xx}$ be the paired empirical covariance.
    We say $\xx$ is \emph{$\lambda_0$-well-concentrated} if $\Sigma_{\xx}$ is invertible and, for all $i, j\in [n]$, $\norm{\xx_i - \xx_j}_{\Sigma_{\xx}}^2\le 2\lambda_0$.
\end{definition}

\begin{observation}[Always $\pass$ on Well-Concentrated Inputs]\label{obs:pass_on_good_data}
    Fix $\lambda_0>0$. 
    If data set $\xx\in \bbR^{n\times d}$ is $\lambda_0$-well-concentrated and $n\ge \inputthreshold$ (so we do not fail immediately), then Algorithm~\ref{alg:main} 
    computes $\score_1=\score_2=0$, passes deterministically, and returns a sample from $\cN(\mu_{\xx}, c^2\Sigma_{\xx})$, where $\mu_{\xx}$ is the empirical mean, $\Sigma_{\xx}$ is the paired empirical covariance, and $c$ is the noise scale set in Algorithm~\ref{alg:main}.
\end{observation}

The essential fact we use about our subgaussian input distributions is that they concentrate in exactly this way.
For definitions and concentration inequalities about subgaussian distribution, see Section~\ref{app:subgaussian_facts}.
The proof of Claim~\ref{claim:gaussian_well_concentrated}, which we omit, is straightforward from these statements.
Recall that the Gaussian distribution $\cN(\mu,\Sigma)$ is subgaussian with parameter $K=O(1)$.
\begin{claim}[Subgaussian Data are Well-Concentrated]\label{claim:gaussian_well_concentrated}
    Let $\xx \in\bbR^{n\times d}$ be drawn i.i.d.\ from a subgaussian distribution $\cD$ with parameter $K_{\cD}$, arbitrary mean, and full-rank covariance.
    There exist constants $K_1$ and $K_2$ such that, when $\lambda_0\ge K_1 K_{\cD}^2 (d + \log n/\beta)$ and $n\ge K_2 K_{\cD}^2 (d + \log 1/\beta)$, then $\xx$ is $\lambda_0$-well-concentrated with probability at least $1-\beta$.
\end{claim}

We now prove our main accuracy claim. 
The calculation is standard: with high probability the true mean is close to the empirical mean, which is close to our private estimate.
\begin{lemma}[Main Accuracy Claim]\label{lemma:accuracy}
    Fix $\eps,\delta\in(0,1)$, $\delta\in (0,\eps/10]$, and $n,d\in \bbN$.
    \accuracystatement
\end{lemma}
\begin{proof}
    In order to not fail immediately, we require
    \begin{align}
        n\ge \inputthreshold \gtrsim K_\cD^2 \cdot \frac{\paren{d + \log n/\beta} \log 1/\delta }{\eps}.
    \end{align}
    This expression has $n$ on both sides; using the fact that $\frac{n}{\log n}\gtrsim \eta$ implies $n\gtrsim \eta \log \eta$, we can rewrite this as
    \begin{align}
            n\gtrsim K_{\cD}^2 \cdot\frac{\log 1/\delta}{\eps} \paren{d+ \log\paren{ \frac{K_{\cD} \log 1/\delta}{\eps \beta}}},
    \end{align}
    as we require.
    In particular, this implies $n\gtrsim K_{\cD}^2\paren{d + \log 1/\beta}$.
    
    Beyond this, we are concerned about four bad events: (i) the data set is not $\lambda_0$-well-concentrated, (ii) the empirical mean is far from the true mean, (iii) the empirical covariance is unlike the true covariance, and (iv) the noise added for privacy is large.
    We will show each of these events happens with probability at most $\beta/4$, analyze the accuracy under the assumption that none of them happens, and use a union bound to finish the proof.

    By Observation~\ref{obs:pass_on_good_data} and Claim~\ref{claim:gaussian_well_concentrated}, with probability at least $1-\beta/4$ data set $\xx$ is $\lambda_0$-well-concentrated and Algorithm~\ref{alg:main} releases $\tilde\mu\sim\cN(\mu_{\xx},c^2\Sigma_{\xx})$, where $\mu_{\xx}$ is the empirical mean, $\Sigma_{\xx}$ is the paired empirical covariance, and 
    \begin{align}
        c^2 &= \squarednoisescale \\
            &\lesssim K_\cD^2 \paren{d + \log n/\beta} \cdot \frac{\log 1/\delta}{\eps^2n^2}.
            \label{eq:value_of_c}
    \end{align}
    We add and subtract the empirical mean and apply the triangle inequality:
    \begin{align}
        \norm{\mu -\tilde\mu}_{\Sigma} &= \norm{\mu -\mu_{\xx} + \mu_{\xx}- \tilde\mu}_{\Sigma} \\
            &\le \norm{\mu -\mu_{\xx}}_{\Sigma} + \norm{\mu_{\xx} -\tilde\mu}_{\Sigma}.
    \end{align}
    
    $\mu_{\xx}$ is the empirical mean, so by Claim~\ref{claim:concentration_of_norm}, with probability at least $1-\beta/4$, we have 
    \begin{align}
        \norm{\mu - \mu_{\xx}}_{\Sigma} \lesssim  K_\cD \sqrt{\frac{d + \log 1/\beta}{n}}.
        \label{eq:empirical_mean_is_close}
    \end{align}

    Since $n\gtrsim K_\cD^2\paren{ d + \log 1/\beta}$, Claim~\ref{claim:concentration_of_covariance} implies that, with probability at least $1-\beta/4$, we have $\norm{v}_{\Sigma}\lesssim \norm{v}_{\Sigma_{\xx}}$ for all vectors $v$.
    Thus $\norm{\mu_{\xx} - \tilde\mu}_{\Sigma}\lesssim \norm{\mu_{\xx} - \tilde\mu}_{\Sigma_{\xx}}$
    This is the ``correct'' norm in which to control $\mu_{\xx} - \tilde\mu$, since $\tilde\mu \sim \cN(\mu_{\xx}, c^2\Sigma_{\xx})$.
    Abusing notation and conflating distributions with random variables, we have
    \begin{align}
        \norm{\mu_{\xx} - \tilde\mu}_{\Sigma_{\xx}} &= \norm{\Sigma_{\xx}^{-1/2} \paren{\cN(0, c^2\Sigma_{\xx})}}_2 %
            = c\cdot \norm{\cN(0,\bbI)}_2.
    \end{align}
    Again applying Claim~\ref{claim:concentration_of_norm}, with probability at least $1-\beta/4$ we have $\norm{\mu_{\xx} - \tilde\mu}_{\Sigma_{\xx}}\lesssim c \sqrt{d + \log 1/\beta}$.
    Plugging in the value of $c$ from Equation~\eqref{eq:value_of_c} and using $d + \log n/\beta \ge d + \log 1/\beta$, we have
    \begin{align}
        \norm{\mu_{\xx} - \tilde\mu}_{\Sigma_{\xx}}
            &\lesssim  K_\cD \sqrt{d + \log n/\beta} \cdot \frac{\sqrt{\log 1/\delta}}{\eps n}\cdot \sqrt{d + \log 1/\beta} \\
            &\lesssim K_\cD\cdot  \paren{d + \log n/\beta} \cdot \frac{\sqrt{\log 1/\delta}}{\eps n}.
    \end{align}
    Combining this with the upper bound on $\norm{\mu - \mu_{\xx}}_{\Sigma}$ in Equation~\eqref{eq:empirical_mean_is_close} finishes the proof.
\end{proof}

\subsection{Running Time}\label{sec:running_time}

For modest privacy parameters, our running time is dominated by a few matrix operations related to computing the empirical covariance and rescaling the data.

\begin{lemma}\label{lemma:running_time}
    Algorithm~\ref{alg:main} can be implemented to require at most 
    \begin{enumerate}
        \item One product of the form $A^TA$ for $A\in \bbR^{n\times d}$; \label{item:matrix_products_1}
        \item Two products of the form $A B$ for $A\in \bbR^{n\times d}$ and $B\in \bbR^{d\times d}$; \label{item:matrix_products_2}
        \item One inversion of a positive definite matrix in $\bbR^{d\times d}$ to polynomial precision (logarithmic bit complexity); and \label{item:inversion}
        \item Further computational overhead of $\tilde{O}\paren{nd/\eps}$.
    \end{enumerate}
\end{lemma}

All of our algorithms can be implemented with finite-precision arithmetic in the word RAM model, setting the bit complexity to be logarithmic in the other parameters.
Informally, matrix products of the form in (\ref{item:matrix_products_1}) or (\ref{item:matrix_products_2}) can be executed in time $\tilde{O}(nd^{\omega-1})$ and inversion can be accomplished in time $\tilde{O}(d^\omega)$, for an overall running time of
\begin{align}
    \tilde{O}\paren{nd^{\omega-1}+  nd/\eps}.
\end{align}
We have used $n= \Omega(kd)=\Omega(d)$, where $k=\Theta(\log(1/\delta)/\eps)$ is our discretization parameter.
These relationships are required by our privacy argument.
In what follows we ignore the precision of our operations, as an exact understanding of the running time is not a focus of our work.
Also note that the asymptotically fastest algorithms for matrix operations are not practical; ``schoolbook'' matrix operations can be performed with the number $3$ in place of $\omega$.

\paragraph{Covariance Estimation}
We compute the covariance $\hat\Sigma = \frac{1}{n} \yy^T \yy$ and its inverse $\hat\Sigma^{-1}$.
For each (paired) point $\yy_i\in[\m]$ we rescale it: $\tilde{\yy}_i \gets \hat\Sigma^{-1}\yy_i$. 
This operation can be written as $\tilde{\yy} \gets y \hat\Sigma^{-1}$, where $y\in\bbR^{n\times d}$ and $\hat\Sigma^{-1}\in\bbR^{d\times d}$. 
(Alternatively, one may solve the system $\hat\Sigma \tilde\yy^T = y^T$ directly.)
Computing each squared Mahalanobis norm (and thus identifying any outliers) can be computed from $y$ and $\tilde\yy$ (via entrywise multiplication and summation) in time $\tilde{O}(nd)$.

A naive implementation of $\mathtt{LargestGoodSet}$ would repeat the above procedure from scratch after finding an outlier.
On top of this, $\stablecovariance$ calls $\mathtt{LargestGoodSet}$ $2k+1$ times to generate
 $\braces{S_{\ell}}_{\ell=0}^{2k}$, the family of sets where $S_\ell$ is the largest $e^{\ell/k}\lambda_0$-good subset for $\yy$.
To improve upon this naive approach, we note three points. 
\begin{enumerate}
    \item Once we have found an outlier, removing it requires a rank-one update to the Mahalanobis norms (in order to check for new outliers). 
        There are well-known efficient techniques for updates of this type, taking $\tilde{O}(nd + d^2) = \tilde{O}(nd)$ time. 
        We provide details below.
    \item For any $\ell\le \ell'$ we know from Lemma~\ref{lemma:family_good_sets} that $S_\ell\subseteq S_{\ell'}$. 
        Therefore, by proceeding through the outlier thresholds in \emph{decreasing} order, each outlier we find can be ignored going forward.
    \item Once we have found an $\ell^*$ such that $\abs{S_{\ell^*}} \le \m - k$, we can immediately halt and either return $\fail$ or output the covariance estimate.
        Recall that we compute 
        \begin{align}
            \score = \min\braces{k, \min_{0\le \ell\le k}\braces{ \m - \abs{S_{\ell}} + \ell}}.
        \end{align}
        If $\ell^*\ge k+1$, then we know that $\score=k$, since for all $\ell\le k$ we have $\abs{S_\ell}\le \abs{S_{k+1}}\le n- k$.
        This means the call to $\PTR$ will fail with probability one.
        If $\ell^*< k+1$, we can still move to the next step: for all $\ell \le \ell^*$ we have $\m - \abs{S_{\ell}} + \ell \ge k$, so the size of the smaller $S_\ell$'s will not affect our score computation.
\end{enumerate}
Combined, these facts imply that we need to perform at most $k$ outlier removals, updating all $n$ Mahalanobis distances each time.
Tracking $\Sigma$ across removals, we can also compute the final weighted covariance.

    Let us discuss the process of updating after an outlier removal in more detail.
    Suppose we want to recompute the norm of a point $v$ after removing $u$ from the covariance matrix.
    Let $\Sigma_1$ be the initial covariance matrix and $\Sigma_2=\Sigma_1 -\frac{1}{\m}u u^T$ the matrix after removing $u$.
    We want to compute $v^T \paren{\Sigma_2}^{-1} v$ for each point in the data set.
    The Sherman-Morrison formula tells us
    \begin{align}
        \Sigma_2^{-1} = \paren{\Sigma_1 - \frac{1}{\m}uu^T}^{-1} = \Sigma_1^{-1} + \frac{\Sigma_1^{-1}uu^T \Sigma_1^{-1}}{\m - u^T \Sigma_1^{-1} u} \label{eq:sherman_morrison_update}
    \end{align}
    Therefore
    \begin{align}
        v^T\Sigma_2^{-1}v &= v^T\Sigma_1^{-1}v + v^T\paren{\frac{\Sigma_1^{-1}uu^T \Sigma_1^{-1}}{\m - u^T \Sigma_1^{-1} u}}v \\
        &= v^T\Sigma_1^{-1}v + \frac{\paren{v^T\Sigma_1^{-1}u}^2}{\m - u^T \Sigma_1^{-1} u}.
    \end{align}
    We have already calculated $v^T\Sigma_1^{-1}v$ and $u^T \Sigma_1^{-1} u$. 
    We can calculate $\Sigma_1^{-1} u$ once (in time $\tilde{O}(d^2)$), which allows us to compute $\paren{v^T \Sigma_1^{-1}u}^2$ in time $\tilde{O}(d)$ for any vector $v$.
    Observe furthermore that Equation~\eqref{eq:sherman_morrison_update} allows us to compute $\Sigma_2^{-1}$ directly in time $\tilde{O}(d^2)$.
    
    We have established that each update after an outlier's removal can be accomplished in time $\tilde{O}(d^2 + nd) = \tilde{O}(nd)$.
    Since we perform at most $k$ updates, this takes time $\tilde{O}(knd)$.
    
    Note that we can simultaneously calculate the final weighted covariance with an additional cost of $\tilde{O}(kd^2)= \tilde{O}(knd)$ time.
    Furthermore, we can return the weighted inverse; it also requires at most $k$ rank-one updates to the original $\paren{\frac 1 \m y^T y}^{-1}$ we computed.

Thus we can execute $\stablecovariance$ with one computation of $\hat\Sigma=y^T y$, one matrix inversion to obtain $\hat\Sigma^{-1}$, one product $y \hat\Sigma^{-1}$, and additional overhead of $\tilde{O}(knd)=\tilde{O}(nd/\eps)$.

\paragraph{Mean Estimation}
As with covariance estimation, naively repeating $\mathtt{LargestCore}$ all $2k+1$ times would waste effort.
Instead, suppose we had a matrix $D \in \bbR^{n\times \Rsize}$, where $\Rsize = \abs{\R}=\Theta(k + \log n/\delta)$, and $D$ has entries $D_{i,j} = \norm{\xx_i - R_j}_{\hat\Sigma}^2$. 
Observe that, with this matrix $D$ in hand, one can compute the family of cores $\braces{S_\ell}_{\ell=0}^{2k}$ and thus the score and weights in time $\tilde{O}(n\Rsize) = \tilde{O}\paren{n/\eps}$.
Given the weights, computing the mean takes time $\tilde{O}(nd)$.

We must compute this matrix $D$ of squared distances.
As we pointed out, $\stablecovariance$ can return $\hat\Sigma^{-1}$, the inverse of the final weighted covariance.
We rescale the data: $\tilde\xx \gets \xx\hat\Sigma^{-1}$. 
With $\tilde\xx$ and $\xx$ in hand, each entry in $D$ can be computed in time $O(d)$.

Thus we can execute $\stablemean$ with one product $x\hat\Sigma^{-1}$ and $\tilde{O}(nd/\eps)$ time.

\paragraph{Main Algorithm}
We can draw from $\cN(\hat\mu, c^2\hat\Sigma)$ in time $\tilde{O}(nd)$ by 
drawing $z\sim\cN(0,\bbI_m)$ and returning $\hat\mu + y^T W^{1/2} z$,
where $y\in \bbR^{\m\times d}$ is the matrix of paired examples and $W$ is the diagonal matrix whose entries are the weights computed by $\stablecovariance$.
(Digital computers do not sample exactly from Gaussian distributions, but we expect one can discretize appropriately, as in~\cite{CanonneKS20discreteGauss}.) 

To see that this is the correct distribution, note that $\hat\Sigma = y^T W y$ by definition and write out the expectation of $\paren{y^T W^{1/2} z}\paren{y^T W^{1/2} z}^T$.

\fi

\section{Stable Covariance Estimation}\label{sec:covariance}

In this section we present $\stablecovariance$, our deterministic algorithm for nonprivate covariance estimation.
Crucially, its parameter estimate is stable when the score it computes is small.
\ifCOLT
\else
    \begin{lemma}[Restatement of Lemma~\ref{lemma:covariances_close}]\label{lemma:covariances_close_two}
        \covariancesclose
    \end{lemma}
\fi

We can prove Lemma~\ref{lemma:covariances_close_two} quickly once we assemble the following definition and lemmas.
The first step in Algorithm~\ref{alg:stable_good_set} is a standard trick to center the data: 
if $x_i,x_j\simiid\cN(\mu,\Sigma)$ then $\frac{1}{\sqrt{2}}\paren{x_i-x_j}\sim\cN(0,\Sigma)$.
(A similar statement holds for subgaussian distributions.)
With this in mind, the lemmas in this section will concern data sets $\yy$ of size $\m= \floor{n/2}$, where $y_i = \frac{1}{\sqrt{2}}\paren{x_i - x_{i+\floor{n/2}}}$.
We emphasize that this section is part of the \emph{privacy} analysis and does not assume the data is in fact mean-zero.
We analyze accuracy in \SectionAppendix~\ref{sec:accuracy}
\begin{definition}[Good Weighting/Subset]\label{def:good_weighting}
    Let $\lambda>0$.
    For data set $\yy=\yy_1,\ldots,\yy_\m$, a vector $w\in [0,1]^\m$ is a \emph{$\lambda$-good weighting of $\yy$} if, for $\Sigma_w = \sum_{j=1}^{\m} w_j\cdot \yy_i \yy_i^T$,
    \begin{align}
        \forall i \in \mathrm{supp}(w), \quad \norm{\Sigma_w^{-1/2} \yy_i }_2^2 \le \lambda.
    \end{align}
    If the weights take the value $\frac{1}{\m}$ over a set $S\subseteq [\m]$ and are zero elsewhere, we will call $S$ a \emph{$\lambda$-good subset of $\yy$}.
\end{definition}
The weights can be thought of as a distribution over data points; our analysis is simpler if we do not require them to sum to one.
Note that this definition requires invertibility of the covariance matrix.

Our proof of Lemma~\ref{lemma:covariances_close_two} requires the following lemma, which shows that if adjacent datasets $\yy,\yy'$ have ``similar'' good weightings $w,v$ respectively, then the corresponding weighted covariances are close in spectral and trace distances.
Nearly identical calculations appeared in BGSUZ; for completeness we prove our version (restated as Lemma~\ref{lemma:identifiability_good_set_restated}) in Appendix~\ref{app:deferred}.

\newcommand{\IdentifiabilityGoodSet}{
Let $\lambda>0$ and let $\yy$ and $\yy'$ be adjacent data sets of size $\m$ that differ in index $i^*$.
    Let $w$ be a $\lambda$-good weighting for $\yy$ and let $v$ be a $\lambda$-good weighting for $\yy'$ with $\sum_{i\neq i^*}\abs{w_i - v_i} \le \rho$ and $\abs{w_{i^*}},\abs{v_{i^*}} \le \eta$.
    Let $\Sigma_w$ and $\Sigma_v$ be the corresponding weighted covariances, as in Definition~\ref{def:good_weighting}.
    Let $\psdjump = \lambda(\rho + \eta)$.%
    
    Then $\Sigma_w,\Sigma_v\succ 0$ and
        $(1-\psdjump)\Sigma_w \preceq \Sigma_{v} \preceq \frac{1}{1-\psdjump} \Sigma_w$.
    Furthermore, if $\psdjump < 1$, then
    \begin{align}
        \norm{\Sigma_v^{-1/2} \Sigma_{w} \Sigma_v^{-1/2} - \bbI }_{\tr}, \norm{\Sigma_w^{-1/2} \Sigma_{v} \Sigma_w^{-1/2} - \bbI }_{\tr} \le \frac 3 2 \cdot \frac{\psdjump}{1-\psdjump}.
    \end{align}

    Finally, if the data sets are the same ($\yy=\yy'$), then the same statements hold with $\psdjump = \gamma_2 = \lambda \cdot \norm{w-w'}_1$.
}
\begin{lemma}\label{lemma:identifiability_good_set}
    \IdentifiabilityGoodSet
\end{lemma}

This lemma, and its mean-estimation analog Lemma~\ref{corollary:identifiability_core}, are statements about ``identifiability.''
Informally, identifiability is used in robust statistics via a standard blueprint: if $\yy$ is a ``good'' data set (for some notion of ``good'') and $\yy'$ is a corruption of $\yy$, then finding a ``good'' data set $z$ that has high overlap with $\yy'$ means $z$ will also have high overlap with $\yy$, and ``goodness'' will ensure that empirical parameter estimates based on $z$ therefore will be accurate for $\yy$.

Following BGSUZ, we use identifiability to establish privacy: on adjacent data sets we will produce ``good'' vectors of weights that are very similar. 
These will ultimately give rise to indistinguishable Gaussian distributions, from which we produce our final mean estimates.

To apply Lemma~\ref{lemma:identifiability_good_set}, we must argue that
$\stablecovariance$ computes good weightings.
\begin{lemma}\label{lemma:output_good_weighted_set}
    Fix data set $\yy\in\bbR^{\m\times d}$, outlier threshold $\lambda_0>0$, and discretization parameter $k\in\bbN$.
    Assume $\m \ge  10 k\lambda_{0}$.
    If $\stablecovariance$ computes $\score <k$, then the weights $\braces{w_i}_{i\in[\m]}$ it produces are a $\frac{25}{18}\cdot \lambda_{0}$-good weighting of $\yy$.
\end{lemma}

Finally, to prove Lemma~\ref{lemma:covariances_close_two} we must establish that on adjacent datasets, $\stablecovariance$ computes weightings which are close in $\ell_1$ and not too large in $\ell_\infty$.
\begin{lemma}\label{lemma:covariance_weights_close}
    Fix data set size $\m$, outlier threshold $\lambda_0>0$, and discretization parameter $k\in\bbN$.
    Assume $\m \ge 10 k \lambda_{0}$.
    Let $\yy$ and $\yy'$ be adjacent data sets of size $\m$ that differ on index $i^*$. 
    Let $w$ be the weights computed by Algorithm~\ref{alg:stable_good_set} on $\yy$ and let $w'$ be the weights computed on $\yy'$. 
    Assume in both cases we compute $\score < k$.
    Then $\sum_{i\neq i^*} \abs{w_i - w_i'} \le 1/\m$.
    Furthermore, $\abs{w_{i^*}}, \abs{w_{i^*}'}\le \frac{1}{\m}$.
\end{lemma}

\begin{algorithm2e}\label{alg:stable_good_set}
    \SetAlgoLined
    \SetKwInOut{Input}{input}
    
    \Input{data set $\xx=(\xx_1,\ldots, \xx_n)\in \bbR^{n\times d}$, outlier threshold $\lambda_0$, discretization parameter $k\in\bbN$}
    \BlankLine
    $m\gets \lfloor n/2\rfloor$\;
    $\forall i\in [m], \yy_i \gets \frac{1}{\sqrt{2}}\paren{\xx_i - \xx_{i+m}}$\tcc*[r]{pair and rescale}
    \BlankLine
    \For{$\ell=0,1,\ldots,2k$}{
        $S_{\ell}\gets \mathtt{LargestGoodSubset}(\yy, \lambda_\ell)$\tcc*[r]{See Alg~\ref{alg:largest_good_set} and Eq~\eqref{eq:cov:lambda_defs} }
    }
    $\score \gets \min\braces{k, \min_{0\le \ell\le k}\braces{\m - \abs{S_\ell} + \ell}}$\;
    \For{$i=1,\ldots,\m$}{
        $w_i \gets \frac{1}{k\m}\sum_{\ell=k+1}^{2k} \indicator{i\in S_\ell}$\;
    }
    $\hat\Sigma \gets \sum_{i\in[\m]} w_i \cdot \yy_i\yy_i^T$\;
    \KwRet $\hat\Sigma$, $\score$\;
    
    \caption{$\stablecovariance(\xx,\lambda_0,k)$, for nonprivate covariance estimation}
\end{algorithm2e}

We are now ready to prove the main result of this section.
\begin{proof}[Proof of Lemma~\ref{lemma:covariances_close_two}]
    On adjacent $\xx$ and $\xx'$, we pair and rescale to obtain $\yy$ and $\yy'$ (also adjacent).
    Lemma~\ref{lemma:covariance_weights_close} says that we get vectors $w$ and $w'$ such that $\sum_{i\neq i^*}\abs{w_i - w_i'}\le 1/\m$ and both $\abs{w_{i^*}}$ and $\abs{w_{i^*}'}$ are at most $\frac{1}{\m}$.
    By Lemma~\ref{lemma:output_good_weighted_set}, both $w$ and $w'$ are $\frac{25}{18}\cdot \lambda_{0}$-good weightings of their respective data set.
    \majoredit
    Setting $\rho=\eta = 1/\m$ and $\lambda=\frac{25}{18}\cdot \lambda_{0}$, Lemma~\ref{lemma:identifiability_good_set} gives us
    \begin{equation}
        \psdjump = \lambda (\rho + \eta) \le \frac{25 \lambda_0}{18} \cdot \frac{2}{m} = \psdjumpdef,
    \end{equation}
    applying $\m= \frac n 2$.
    Lemma~\ref{lemma:identifiability_good_set} also lets us bound the trace norms with 
    \begin{equation}
        \frac 3 2 \cdot \frac{\psdjump}{1-\psdjump} \le \frac{25}{3}\cdot\frac{1}{1-\psdjump}\cdot\frac{\lambda_0}{n}. 
    \end{equation}
    This concludes the proof.
    \editdone
\end{proof}

\ifCOLT

\else
    In the remainder of this section, we prove Lemma~\ref{lemma:identifiability_good_set}, Lemma~\ref{lemma:output_good_weighted_set}, and Lemma~\ref{lemma:covariance_weights_close}.
    Finally, in Section~\ref{sec:cov:sensitivity}, we prove that the score returned by $\stablecovariance$ is low-sensitivity.

\begin{algorithm2e}\label{alg:largest_good_set}
    \SetAlgoLined
    \SetKwInOut{Input}{input}
    
    \Input{data set $\yy=(\yy_1,\ldots, \yy_\m)\in \bbR^{\m\times d}$, outlier threshold $\lambda$}
    \tcc{For vector $v$ and singular matrix $A$, define $\norm{A^{-1} v}_2 = + \infty$}
    \BlankLine
    $S \gets [\m]$\;
    \Repeat{$\mathtt{OUT} = \emptyset$}{
        $\mathtt{OUT} \gets \braces{i \in S : \norm{\paren{\frac{1}{\m} \sum_{j\in S} \yy_j \yy_j^T}^{-1/2} \yy_i }_2^2 > \lambda}$\;
        $S \gets S \setminus \mathtt{OUT}$
    }
    \KwRet $S$\;
    \caption{$\mathtt{LargestGoodSubset}(\yy,\lambda)$, subroutine for covariance estimation}
\end{algorithm2e}

\subsection{Families of Largest Good Subsets}

Before proceeding to the proofs of Lemmas~\ref{lemma:output_good_weighted_set} and~\ref{lemma:covariance_weights_close}, we establish a few facts about the operation of Algorithm~\ref{alg:stable_good_set}.

Algorithm~\ref{alg:stable_good_set}, $\stablecovariance$, takes as input a discretization parameter $k\in\bbN$.
For each $\ell=0,1,\ldots,2k$, 
\majoredit
it computes an outlier threshold
\begin{equation}
    \lambda_\ell = \lambda_0 \cdot \frac{1}{1-\ell/10k}.
    \label{eq:cov:lambda_defs}
\end{equation}
This is approximately $e^{\ell/10k}\lambda_0$; we can see that these thresholds grow slowly in the $\ell\le 2k$ regime we consider.
\editdone

For each outlier threshold, $\mathtt{StableCovariance}$ calls $\mathtt{LargestGoodSubset}$ to find a set $S_{\ell}\subseteq [m]$, the largest $\lambda_\ell$-good subset.
Denote by $\braces{S_{\ell}}_{\ell=0}^{2k}$ this family of subsets.
In this subsection we prove that each $S_\ell$ is a unique largest good subset and discuss some properties of families of sets of this form.
\begin{lemma}\label{lemma:unique_largest}
    For any data set $\yy$ and $\lambda>0$, there is a largest $\lambda$-good subset $S^*$ that is unique and satisfies the following property: if $S_0$ is a $\lambda$-good subset, then $S_0\subseteq S^*$.
    Furthermore, Algorithm~\ref{alg:largest_good_set} returns $S^*$.
\end{lemma}
\begin{proof}
    Observe that Algorithm~\ref{alg:largest_good_set} either returns a $\lambda$-good subset or the empty set.
    We will show that, if $\yy$ has a $\lambda$-good set $S_0$, then the set returned by Algorithm~\ref{alg:largest_good_set} contains $S_0$.
    This finishes the proof: for any two distinct $\lambda$-good sets, the output contains their union and thus must be strictly larger than both.

    Suppose $\yy$ has a $\lambda$-good set $S_0$.
    We prove that $S_0 \subseteq S$ is an algorithmic invariant, where $S$ denotes the set operated upon in Algorithm~\ref{alg:largest_good_set}.
    At initialization we have $S_0\subseteq S=[\m]$.
    Pick an index $j\in S_0$ and consider a value of $S$ at the start of the \textbf{repeat-until} block of Algorithm~\ref{alg:largest_good_set}.
    We will show that $j$ is not added to $\mathtt{OUT}$.
    Let $\Sigma_S = \frac{1}{\m}\sum_{i\in S}\yy_i \yy_i^T$ and similarly define $\Sigma_{S_0}$.
    
    Removing points cannot increase the matrix in the positive semidefinite order, so
    \begin{align}
        \Sigma_S = \frac{1}{\m}\sum_{i\in S} \yy_i \yy_i^T \succeq \frac{1}{\m} \sum_{i\in S_0} \yy_i\yy_i^T = \Sigma_{S_0}.
    \end{align}
    Since the $\Sigma_{S_0}$ is invertible (by the goodness assumption), $\Sigma_S$ is as well. 
    So we also know that $\Sigma_S^{-1}\preceq \Sigma_{S_0}^{-1}$, and thus in turn
    \begin{align}
        \norm{ \yy_j }_{\Sigma_{S}}^2 \le 
            \norm{ \yy_j }_{\Sigma_{S_0}}^2.
    \end{align}
    The right term is at most $\lambda$ by the $\lambda$-goodness assumption on $S_0$, so $j$ is not removed as an outlier.

    To finish, note that Algorithm~\ref{alg:largest_good_set} only terminates when it has found a $\lambda$-good subset.
\end{proof}

The following lemma shows that families of good subsets computed on \emph{adjacent} data sets are closely interwoven.
\begin{lemma}\label{lemma:family_good_sets}
    Let $\lambda_0>0$ and $k\in\bbN$. 
    Let $\yy$ and $\yy'$ be adjacent data sets of size $m$ that differ in index $i^*$.
    Let $\braces{S_{\ell}}_{\ell=0}^{2k}$ be the family of sets where $S_\ell$ is the largest $\lambda_\ell$-good subset for $\yy$ (see Equation~\eqref{eq:cov:lambda_defs}).
    Let $\braces{T_{\ell}}_{\ell=0}^{2k}$ be the corresponding family of largest $\lambda_\ell$-good sets computed on $\yy'$.
    Assume $m \ge 10 k \lambda_{0}$.
    We have the following properties:
    \begin{enumerate}
        \item[(i)] For any $0\le \ell\le \ell'\le 2k$, we have $S_\ell\subseteq S_{\ell'}$.
        \item[(ii)] For any $0\le \ell < 2k$, we have
            $S_\ell\setminus\braces{i^*}\subseteq T_{\ell+1}$.
    \end{enumerate}
\end{lemma}

Before we prove Lemma~\ref{lemma:family_good_sets}, we will prove a statement relating to the stability of good subsets (on a single data set $\yy$) under the removal of a index.
If we remove an index $j$ from a $\lambda$-good subset $S$, the covariance may shrink and the resulting subset may no longer be $\lambda$-good.
However, we can show that the resulting set is $\lambda'$-good, where $\lambda'$ is larger than $\lambda$ by a multiplicative factor that depends on $\lambda$ and $\m$.
\begin{claim}\label{claim:remove_one_from_good_subset}
    Let $\lambda>0$ and $m > \lambda$.
    Let $\yy$ be a data set of size $\m$ and let $S$ be a $\lambda$-good subset for $\yy$.
    For any $j\in[\m]$, $S\setminus\braces{j}$ is a $\frac{1}{1-\lambda/m}\cdot\lambda$-good subset for $\yy$.
\end{claim}
\begin{proof}
    Let $S' = S\setminus\braces{j}$.
    Assume $j\in S$, otherwise we are done.
    By construction, we have $\Sigma_{S'} = \Sigma_S - \frac{1}{\m} \yy_j\yy_j^T$.
    We claim that $\Sigma_{S'} \succeq \Sigma_S - (\lambda/\m) \Sigma_S$.
    This is equivalent to $\frac{1}{\m} \yy_j\yy_j^T \preceq (\lambda/\m) \Sigma_S$, or
    \begin{align}
        \Sigma_S^{-1/2} \yy_j\yy_j^T \Sigma_S^{-1/2} \preceq \lambda \bbI.
    \end{align}
    This holds because $\norm{\Sigma_S^{-1/2} \yy_j}_2^2 \le \lambda$ (since $j\in S$).
    By assumption $\frac{\lambda}{\m}<1$, so
    \begin{align}
        \Sigma_{S'}^{-1} &\preceq \frac{1}{1-\lambda/\m} \cdot \Sigma_S^{-1} 
    \end{align}
    For any $i\in S'\subseteq S$ we had $\norm{\yy_i}_{\Sigma_S}^2\le \lambda$, so we also have $\norm{\yy_i}_{\Sigma_{S'}}^2 \le \frac{1}{1-\lambda/m}\cdot \lambda$.
    \editdone
\end{proof}

\begin{proof}[Proof of Lemma~\ref{lemma:family_good_sets}]
    \majoredit
    To prove statement $(i)$, pick two indices $\ell\le \ell'$.
    We know that $S_\ell$ is a $\lambda_{\ell'}$-good subset for $\yy$, because increasing the outlier threshold cannot cause any points in $S_\ell$ to become outliers.
    By construction, $S_{\ell'}$ is the largest $\lambda_{\ell'}$-good subset for $\yy$.
    What's more, by Lemma~\ref{lemma:unique_largest}, $S_{\ell'}$ contains all $\lambda_{\ell'}$-good subsets.
    In particular, $S_\ell\subseteq S_{\ell'}$.
    
    To prove statement (ii), consider $\smini$, where data sets $\yy$ and $\yy'$ differ only in index $i^*$.
    We know that $S_\ell$ is a $\lambda_{\ell}$-good subset for $\yy$, so by Claim~\ref{claim:remove_one_from_good_subset} we know that $\smini$ is $\lambda'$-good for 
    \begin{align} 
        \lambda' &\le  \frac{1}{1-\lambda_\ell/m}\cdot \lambda_\ell.
    \end{align}
    The value of $\lambda_{\ell+1}$, in Equation~\eqref{eq:cov:lambda_defs}, is set precisely so that $\lambda'\le \lambda_{\ell+1}$.
    (This follows from direct calculation: plug in the definition and apply $\frac{\lambda_0}{m}\le \frac{1}{10k}$.)
    Thus $\smini$ is $\lambda_{\ell+1}$-good.
    \editdone
    
    The next key fact is that, on the indices $i$ in $\smini$, we have $\yy_i=\yy_i'$.
    Therefore $\smini$ is $\lambda_{\ell+1}$-good for $\yy'$.
    Since, by Lemma~\ref{lemma:unique_largest}, $T_{\ell+1}$ contains all such sets, $\smini\subseteq T_{\ell+1}$.
\end{proof}

\subsection{The Weights Are Good (Proofs of Lemma~\ref{lemma:output_good_weighted_set})}

Lemma~\ref{lemma:output_good_weighted_set} says that, when $\stablecovariance$ computes $\score<k$, the weights it produces are in fact a good weighting.

\begin{proof}[Proof of Lemma~\ref{lemma:output_good_weighted_set}]
    For $\ell\in\braces{0,\ldots,2k}$, let $S_\ell$ denote the largest $\lambda_\ell$-good subset for $\yy$ (see Equation~\eqref{eq:cov:lambda_defs}).
    Let $\Sigma_{\ell} = \frac{1}{\m}\sum_{i\in S_\ell} \yy_i \yy_i^T$.
    Algorithm~\ref{alg:stable_good_set} computes weights $w_i = \frac{1}{km}\sum_{\ell=k+1}^{2k} \indicator{i\in S_\ell}$,  meaning that we can rewrite the released $\hat\Sigma$ as 
    \begin{align}
        \hat\Sigma = \sum_{i=1}^{\m} w_i \cdot \yy_i\yy_i^T = \frac{1}{k} \sum_{\ell=k+1}^{2k} \Sigma_{\ell}.
    \end{align}
    The larger sets contain the smaller ones, so we have a lower bound on $\hat\Sigma$ in terms of $\Sigma_{k+1}$:
    \begin{align}
        \hat\Sigma \succeq \frac{1}{k}\sum_{\ell=0}^{2k} \Sigma_{k+1} = \Sigma_{k+1}.
    \end{align}
    Note that all of these sets are $\lambda_{2k}$-good.
    
    Now, any two values $\ell$ and $\ell'$ such that $k+1\le \ell\le \ell'\le 2k$ correspond to sets $S_{\ell}, S_{\ell'}$ that differ in at most $k$ elements: $S_{\ell}$ is a subset of $S_{\ell'}$, $S_{\ell'}$ has size at most $\m$, and $S_{\ell}$ has size at least $\m-k+1>\m-k$.
    (To see this last fact, observe that $\score<k$ means there is an $\ell^*\le k$ such that $n - \abs{S_{\ell^*}} + \ell^* < k$, which implies $\abs{S_{\ell^*}}>n-k$.
    Then note that $S_{\ell^*}\subseteq S_{\ell}$, since $\ell^*\le k \le \ell$.)
    By the identifiability lemma for good sets, Lemma~\ref{lemma:identifiability_good_set}, we have $\Sigma_{\ell} \succeq (1-\gamma) \Sigma_{\ell'}$ for $\gamma = k \lambda_{2k}/\m$. 
    This in turn implies an upper bound on the squared Mahalanobis distance for $\Sigma_{\ell}$.

    Pick any point $i\in \bigcup_{\ell=k+1}^{2k} S_{\ell}$. 
    We know that $i \in S_{2k}$, so $\norm{\yy_i}_{\Sigma_{2k}}^2 \le \lambda_{2k}$.
    Using our above calculations, we have 
    \begin{align}
        \norm{\yy_i}_{\hat\Sigma}^2 &\le \norm{\yy_i}_{\Sigma_{{k+1}}}^2 \\
            &\le \frac{1}{1-k\lambda_{2k}/m}\cdot \norm{\yy_i}_{\Sigma_{2k}}^2 \\
            &\le \frac{\lambda_{2k}}{1 - k\lambda_{2k}/m}.
    \end{align}
    We have $\lambda_{2k}\le \frac 5 4 \lambda_0$ and $\frac{k \lambda_0}{m} \le \frac{1}{10}$ by assumption.
    Plugging these in concludes the proof.
\end{proof}

\subsection{The Weights Are Stable (Proof of Lemma~\ref{lemma:covariance_weights_close})}

Lemma~\ref{lemma:covariance_weights_close} says that, on adjacent inputs, $\stablecovariance$ produces nearly identical weight vectors (assuming the scores are not large).

\begin{proof}[Proof of Lemma~\ref{lemma:covariance_weights_close}]
    For $\ell\in\braces{0,\ldots,2k}$, let $S_\ell$ denote the largest $\lambda_\ell$-good subset for $\yy$ and let $T_\ell$ denote the largest $\lambda_\ell$-good subset for $\yy'$.
    Recall the notation from Algorithm~\ref{alg:stable_good_set}: for each $i$ we compute counts $c_i = \sum_{\ell=k+1}^{2k} \indicator{i\in S_{\ell}}$ and let $c_i +\Delta_i= \sum_{\ell=k+1}^{2k} \indicator{i\in T_{\ell}}$.
    Recall that $\yy$ and $\yy'$ differ in index $i^*$.
    By Lemma~\ref{lemma:family_good_sets}, for all $i\neq i^*$, we know that $\Delta_i\in\braces{-1,0,+1}$.

    We can bound the number of indices $i\neq i^*$ that have nonzero $\Delta_i$.
    Because the score on $\yy$ is less than $k$, the set $S_{k}$ has size at least $\m-k+1$.
    By Lemma~\ref{lemma:family_good_sets}, for every $\ell>k$ we have $S_{k}\setminus\braces{i^*}\subseteq S_{\ell}$ and $S_{k}\setminus\braces{i^*}\subseteq T_{\ell}$, so for all the indices $i\in S_{k}\setminus\braces{i^*}$ we compute weight exactly $\frac{1}{\m}$.
    So the number of indices $i\neq i^*$ with $\Delta_i\neq 0$ is at most $k$.
    Thus
    \begin{align}
         \sum_{i\in [\m]\setminus\braces{i^*}}\abs{w_i - w_i'}
            &=  \sum_{i\in [\m]\setminus\braces{i^*}}\abs{\frac{c_i}{k\m} - \frac{c_i+\Delta_j}{k\m}} \\
            &= \sum_{i\in [\m]\setminus\braces{i^*}}\abs{\frac{\Delta_i}{k\m}} \\
            &\le k\cdot \frac{1}{k\m}.
    \end{align}
    Finally, note that all weights (in particular, $w_{i^*}$ and $w_{i^*}'$) are at most $1/\m$ by construction. 
\end{proof}

\subsection{The Score Has Low Sensitivity}\label{sec:cov:sensitivity}

In this subsection we introduce the function $g(\cdot)$ that $\stablecovariance$ computes to determine the value of $\score$, which it returns.
We prove that $\score$ is low-sensitivity, which will allow us to use it as an input to $\PTR$, our propose-test-release subroutine.

\begin{definition}\label{def:g_scaled_size}
    Fix data set $\yy$, outlier threshold $\lambda_0>0$, and discretization parameter $k\in\bbN$.
    For $\ell\in\braces{0,1,\ldots,k}$, let $S_\ell$ denote the largest $\lambda_\ell$-good subset for $\yy$, with $S_\ell=\emptyset$ if no good subset exists.
    Define
    \begin{itemize}
        \item $f(\yy) \defeq \min_{\ell \in \braces{0,\ldots,k}} \braces{\m -\abs{S_\ell} + \ell}$ and 
        \item $g(\yy) \defeq \min \braces{f(\yy), k}$.
    \end{itemize}
\end{definition}

The function $f(\yy)$ is small when there is $\lambda$ not too much larger than $\lambda_0$ that yields a large set $S_{\lambda}$. 
In particular, when a data set $\yy$ has $S=[\m]$ as a $\lambda_0$-good subset, then $f(\yy)=g(\yy)=0$.

For additional intuition, recall that a key subroutine in BGSUZ was computing the Hamming distance to the space of data sets with no outliers.
Our score function approximates this quantity.
We make this precise in the following claim, which is stated to aid the reader and not directly used elsewhere.
\begin{claim}\label{claim:hamming_distance}
    Let $\mathfrak{S}_{\lambda}$ be the space of size-$\m$ data sets where $[\m]$ is a $\lambda$-good subset. 
    Let $d(y,\mathfrak{S}_{\lambda}) = \min_{z\in\mathfrak{S}_{\lambda}} d(y,z)$ denote the Hamming distance to this space.
    If we compute $g(y) < k$, then we know that $d(y,\mathfrak{S}_{e\lambda_0})\le g(y)\le 2\cdot d (y,\mathfrak{S}_{\lambda_0})$.
\end{claim}
\begin{proof}
    For the upper bound, assume $d(y,\mathfrak{S}_{\lambda_0}) = r$, so there exists a $z$ with $d(y,z)=r$ and $r\in\mathfrak{S}_{\lambda_0}$.
    Removing $r$ points from $z$, we have an $\lambda_r$-good subset of $y$ with size $\m-r$, so $g(y)\le \m - \m + r + r = 2 \cdot d(y,\mathfrak{S}_{\lambda_0})$.

    Now assume $g(y)=r< k$, so there exists some $\ell$ with $\m - \abs{S_\ell}+\ell = r$.
    We know that $S_\ell\subseteq S_k$, so augmenting $S_\ell$ with at most $r$ copies of $\vec{0}$ yields a set $z\in \mathfrak{S}_{e\lambda_0}$ with $d(y,z) =r$.
    Thus $d(y,\mathfrak{S}_{e\lambda_0})\le r = g(y)$.
\end{proof}

The main result of this section says that our score function has low sensitivity.
\begin{lemma}\label{lemma:good_set_sensitivity}
    Fix data set size $\m$, outlier threshold $\lambda_0>0$, and discretization parameter $k\in\bbN$.
    Assume $m \ge 10 k \lambda_{0}$.
    For all adjacent $\yy$ and $\yy'$ we have $|g(\yy) - g(\yy')|\le 2$.
\end{lemma}
We argued in Lemma~\ref{lemma:family_good_sets} that, for certain parameters, if data set $\yy$ has a large good set then adjacent $\yy'$ must have a (slightly \emph{smaller}) good set (for a slightly \emph{larger} outlier threshold).
Under these conditions, then, we have have $f(\yy)\approx f(\yy')$.
These conditions might be violated when $\yy$ has no good sets under modest $\lambda$, causing $f(\yy)$ to be large.
To combat this, we introduce $g(\yy)$, which cannot be larger than $k$.
\begin{proof}[Proof of Lemma~\ref{lemma:good_set_sensitivity}]
    Without loss of generality, assume $g(\yy) \le g(\yy')$.
    We will show that $g(\yy') \le g(\yy) + 2$ by analyzing two cases.

    \textbf{Case 1:} Suppose $g(\yy)=k$.
        Then $g(\yy') \le g(\yy) \le g(\yy)+2$ by construction, since $g(\cdot)$ is capped at $k$.

    \textbf{Case 2:} Suppose $g(\yy)<k$.
        This can only happen when $f(\yy)<k$, so there exists an $\ell^*\in\braces{0,\ldots,k}$ and subset $S_{\ell^*}$ that (i) is $\lambda_{\ell^*}$-good for $\yy$ and (ii) satisfies $\m - \abs{S_{\ell^*}} +\ell^* < k$.
        Furthermore, we know $\ell^*<k$, since $m-\abs{S_{\ell^*}}$ is nonnegative.
        
        We can now apply Lemma~\ref{lemma:family_good_sets}, since we have assumed $\m \ge 10 k \lambda_{0}$.
        For $\ell\in\braces{0,\ldots,k}$, let $T_{\ell}$ be the largest $\lambda_\ell$-good subset for $\yy'$.
        Lemma~\ref{lemma:family_good_sets} says that $S_{\ell^*}\setminus\braces{i^*}\subseteq T_{\ell^*+1}$.
        This allows us to upper bound $g(\yy')$:
        \begin{align} 
            g(\yy') \le f(\yy') &= \min_{\ell \in \braces{0,\ldots,k}} \m - \abs{S_{\ell}} + \ell\\
            &\le \m - \abs{T_{\ell^*+1}} + \ell^* + 1\\
            &\le \m - \paren{|S_{\ell^*}| - 1} + \ell^* + 1 \\
            &= g(\yy) + 2.
    \end{align}
    So we are done.
\end{proof}

\fi

\ifCOLT
\else
    \section{Stable Mean Estimation}\label{sec:mean}
    
In this section we present $\stablemean$, our algorithm for nonprivate mean estimation, and prove its stability guarantee.
It is based on a notion of ``outlyingness'' where inliers are sufficiently close to many other points in the data.
More formally, one may think of a graph with $n$ vertices corresponding to a data set $\xx$, where vertices $i$ and $j$ are adjacent if they are close.
Under this interpretation, outliers are points with degree below some threshold.

$\stablemean$ takes as input a data set $\xx$, a matrix $\hat\Sigma$ that serves as a preconditioner, an outlier threshold $\lambda_0$, a ``discretization'' parameter $k$, and a set of ``reference points'' $\R\subseteq [n]$.
\majoredit
Similar to the series $\lambda_0,\lambda_1,\ldots,\lambda_{2k}$ consider in the previous section, our mean estimation subroutine uses a series of outlier thresholds that starts with the same $\lambda_0$ but grows slightly faster: for all $\ell\ge 0$,
\begin{align}
    \Lambda_\ell \defeq \lambda_0 \cdot \paren{1-\psdjump}^{-\ell},
    \label{eq:mean:Lambda_defs}
\end{align}
where $\psdjump = \psdjumpdef$ (as in Lemma~\ref{lemma:covariances_close_two}).
\editdone

We analyze this algorithm's stability with respect to simultaneously changing a single data point as well as slightly changing the preconditioner $\hat\Sigma$.
The exact meaning of ``slight change'' we use is that of being multiplicatively close in the positive definite order, as in the consequence of Lemma~\ref{lemma:covariances_close}.

The reference points $\R\subseteq [n]$ are used to estimate which points are outliers: rather than check the distance to every one of the other $n-1$ points in the data set, we compute distances to a few randomly chosen data points.
One can think of this as estimating the degree of vertices in the graph defined above: we will argue that the algorithm is private when $\R$ is sufficiently representative.
A standard concentration argument says that a sufficiently large random $\R$ will be representative with probability at least $1-\delta$.
(Such an argument appeared, with a similar application to private estimation, in the work of~\cite{tsfadia2022friendlycore}.) 
\begin{definition}[Degree-Representative]\label{def:degree_representative}
    Fix data set $\xx$, covariance $\Sigma$, outlier threshold $\lambda_0$, and reference set $\R\subseteq [n]$.
    For all $i$, let
    \begin{align}
        N_i &\defeq \braces{j\in[n] : \norm{\xx_i - \xx_j}_{\Sigma}^2\le 4\Lambda_{2k}} \\
        \tilde{N}_i &\defeq \braces{j\in \R : \norm{\xx_i - \xx_j}_{\Sigma}^2\le 4\Lambda_{2k}}.
    \end{align}
    Let $z_i = \abs{N_i}$ and $\tilde{z}_i = |\tilde{N}_i|$ be the sets' sizes.
    We say that $\R$ is \emph{degree-representative for $\xx$} if for every index $i$ we have $\abs{\frac{1}{\abs{\R}}\tilde{z}_i - \frac{1}{n}z_i}\le \frac 1 6$.
\end{definition}

\begin{algorithm2e}\label{alg:stable_mean}
    \SetAlgoLined
    \SetKwInOut{Input}{input}
    
    \Input{data set $\xx=(\xx_1,\ldots, \xx_n)\in \bbR^{n\times d}$, covariance $\hat\Sigma$, outlier threshold $\lambda_0$, discretization parameter $k\in\bbN$, reference set $\R\subseteq [n]$}
    \BlankLine
    \For{$\ell=0,1,\ldots,2k$}{
        $S_{\ell} \gets \mathtt{LargestCore}(\xx,\hat\Sigma, \Lambda_\ell, 
        |R|-\ell,
        \R)$\tcc*[r]{See Alg~\ref{alg:largest_core} and Eq~\eqref{eq:mean:Lambda_defs}}
    }
    $\score \gets \min\braces{ k, \min_{0\le \ell\le k} \braces{n - \abs{S_\ell}+\ell}}$\;
    \For{$i=1,\ldots, n$}{
        $c_i \gets \sum_{\ell=k+1}^{2k} \indicator{i\in S_{\ell}}$\;
    }
    $Z \gets \sum_{i\in[n]} c_i $\tcc*[r]{normalizing constant}
    $\forall i\in[n]$, $w_i \gets \frac{c_i}{Z}$\tcc*[r]{set $w_i\gets 0$ if $Z=0$}
    $\hat\mu \gets \sum_{i\in[n]} w_i \xx_i$\;
    \KwRet $\hat\mu, \score$\;
    \caption{$\stablemean(\xx,\hat\Sigma,\lambda_0,k, \R)$ for nonprivate mean estimation}
\end{algorithm2e}

\begin{algorithm2e}\label{alg:largest_core}
    \SetAlgoLined
    \SetKwInOut{Input}{input}
    
    \Input{data set $\xx=(\xx_1,\ldots, \xx_n)\in \bbR^{n\times d}$, covariance $\hat\Sigma$, outlier threshold $\lambda$, degree threshold $\tau$, reference set $\R\subseteq [n]$}
    \BlankLine
    \For{$i\in[n]$}{
        $N_i = \braces{j\in \R : \norm{\xx_i - \xx_j}_{\Sigma}^2 \le 2\lambda}$\tcc*[r]{Nearby points in $\R$}
    }
    \KwRet $\braces{i\in [n] :  \abs{N_i}\ge \tau}$\;
    \caption{$\mathtt{LargestCore}(\xx, \hat\Sigma, \lambda, \tau, \R)$, subroutine for mean estimation}
\end{algorithm2e}

The main result in this section is the following stability guarantee for our nonprivate mean estimates on adjacent data sets.
\begin{lemma}[Lemma~\ref{lemma:means_close} Restated]\label{lemma:means_close_two}
    \meansclose
\end{lemma}
As in Section~\ref{sec:covariance}, we assemble another definition and a few lemmas before proving Lemma~\ref{lemma:means_close_two}.

In this section ``outliers'' are those points with insufficient degree.
We formalize this notion (without explicit reference to graphs), below.
We consider weighted subsets, as in Section~\ref{sec:covariance}.
Note that, for outlier thresholds $\lambda$, we ask that points have squared distance at most $2\lambda$, to keep the interpretation of $\lambda$ consistent across the paper.
\begin{definition}[Weighted Core]
    Let $\xx$ be a data set of size $n$.
    Let $\lambda>0$, $\tau\in\bbN$, and $\R\subseteq[n]$.
    A distribution $w$ over $[n]$ is a \emph{$(\tau, \lambda, \Sigma)$-weighted-core for $\xx$ with respect to  $\R$} if, for each $i\in \supp{w}$, there is a set $N_i \subseteq \R$ of size at least $\tau$ such that $\forall j\in N_i$, $\norm{\xx_i - \xx_j}_{\Sigma}^2\le 2\lambda$.
    If $w$ is uniform over a subset $S$, we will refer to $S$ as a \emph{$(\tau, \lambda, \Sigma)$-core for $\xx$ with respect to $\R$}.
\end{definition}

\begin{remark}\label{remark:unique_cores}
    For a given data set $\xx$, reference set $\R$, and parameters $\tau, \lambda,$ and $\Sigma$, one can directly find $S^*\subseteq [n]$, the largest $(\tau, \lambda,\Sigma)$-core for $\xx$ with respect to $\R$ (if $\yy_i$ has $\tau$ nearby points in $\R$, put $i$ in $S^*$, otherwise omit it).
    Contrast this with Section~\ref{sec:covariance}, where it was nontrivial that an efficient algorithm recovers the largest $\lambda$-good sets.
    In addition to such an algorithmic guarantee, Lemma~\ref{lemma:unique_largest} established that the largest $\lambda$-good set is unique and contains all other $\lambda$-good sets.
    The analogous observation is obvious for $(\tau,\lambda,\Sigma)$-cores.
\end{remark}

The first statement we need in the proof of Lemma~\ref{lemma:means_close_two} tells us that, if two weighted cores (on adjacent data sets) are close in $\ell_1$ distance, then their means are close.
We give the proof of Corollary~\ref{corollary:identifiability_core} (restated as Corollary~\ref{corollary:identifiability_core_restated}) in Appendix~\ref{app:deferred}; nearly identical calculations appeared in BGSUZ.
\newcommand{\IdenfitiabilityCore}{
    Let $\xx$ and $\xx'$ be adjacent data sets of size $n$ that differ in index $i^*$.
    Let $\lambda>0$, $\tau\in\bbN$, and $\R\subseteq [n]$.
    Let $\covone$ and $\covtwo$ be positive definite matrices satisfying $(1-\gamma)\covone \preceq \covtwo \preceq \frac{1}{1-\gamma}\covone$ for $\gamma<1$.
    Let $w$ be a $(\tau, \lambda, \covone)$-weighted-core for $\xx$ with respect to $\R$ and let $v$ be a $(\tau, \lambda, \covtwo)$-weighted-core for $\xx'$ with respect to $\R$, with $\tau > \frac n 2 + 1$.
    If $\sum_{i\neq i^*}\abs{w_i - v_i} \le \rho$ and $\abs{w_{i^*}},\abs{v_{i^*}}\le \eta $ then 
    \begin{align}
        \norm{\mu_w - \mu_v}_{\covone}^2 = \norm{\sum_{i\in[n]}w_i\xx_i -\sum_{i\in[n]}v_i \xx_i'}_{\covone}^2 \le \frac{\lambda }{1-\gamma} \cdot \paren{\rho + 2\eta}^2.
    \end{align}
}
\begin{corollary}\label{corollary:identifiability_core}
    \IdenfitiabilityCore
\end{corollary}

To apply Corollary~\ref{corollary:identifiability_core}, we need to show that the weights produced by $\stablemean$ are in fact a weighted core for $\xx$.
\begin{lemma}\label{lemma:weights_are_core}
    Fix a set of inputs to $\stablemean$ (Algorithm~\ref{alg:stable_mean}): data set $\xx$, covariance $\hat\Sigma$, outlier threshold $\lambda_0$, discretization parameter $k\in\bbN$, and reference set $\R\subseteq [n]$.
    Assume $\abs{\R}>6k$ and $n > 2 k\lambda_{2k}$.
    If $\stablemean$ computes $\score<k$ and $\R$ is degree-representative for $\xx$, then the weights $w\in [0,1]^n$ it produces are an $\paren{\frac{n}{2} + 1, \Lambda_{2k}, \hat\Sigma}$-weighted-core for $\xx$ with respect to $[n]$.
\end{lemma}

Additionally, the weights produced by $\stablemean$ on adjacent data sets must be close in $\ell_1$ distance.
\begin{lemma}\label{lemma:mean_weights_close}
    Fix outlier threshold $\lambda_0\ge 1$, discretizaion parameter $k\in\bbN$, and reference set $\R\subseteq [n]$.
    Let $\xx$ and $\xx'$ be adjacent data sets of size $n$ that differ in index $i^*$.
    Let $\covone$ and $\covtwo$ be positive definite matrices satisfying $(1-\psdjump)\covone \preceq \covtwo \preceq \frac{1}{1-\psdjump}\covone$ for $\psdjump=\psdjumpdef$. 
    Assume $n \ge 20 k \lambda_{0}$.
    
    Let $w$ be the weights Algorithm~\ref{alg:stable_mean} computes on inputs $(\xx,\covone,\lambda_0,k,\R)$ and $w'$ the weights it computes on inputs $(\xx',\covtwo,\lambda_0,k,\R)$.
    Assume in both cases that it computes $\score<k$.
    \majoredit
    Then
    \begin{align}
        \sum_{i\neq i^*} \abs{w_i - w_i'} \le  \frac{1}{(1-k/n)^2}\cdot \frac{3}{n}.
    \end{align}
    Furthermore $\abs{w_{i^*}},\abs{w_{i^*}'} \le \frac{1}{1-k/n}\cdot\frac{1}{n}$.
    \editdone
\end{lemma}

We are now ready to prove the stability guarantee for the nonprivate mean estimate.
\begin{proof}[Proof of Lemma~\ref{lemma:means_close_two}]
    Let $w$ be the vector of weights computed by $\stablemean$ on input $(\xx,\covone, \lambda_0, k, \R)$ and $w'$ the weights computed on input $(\xx', \covtwo, \lambda_0,k,\R)$.
    By Lemma~\ref{lemma:weights_are_core} and the fact that $\R$ is degree-representative for $\xx$, $w$ is an $(n/2+1, \Lambda_{2k}, \covone)$-weighted-core for $\xx$.
    For the same reasons, $w'$ is an $(n/2+1, \Lambda_{2k}, \covtwo)$-weighted-core for $\xx'$.
    By Lemma~\ref{lemma:mean_weights_close}, $w$ and $w'$ satisfy $\sum_{i\neq i^*}\abs{w_i - w_i'}\le \frac{1}{(1-k/n)^2}\cdot\frac{3}{n}$ and $\abs{w_{i^*}},\abs{w_{i^*}'}\le \frac{1}{1-k/n}\cdot\frac{1}{n}$.
    Thus, setting $\rho = \frac{1}{(1-k/n)^2}\cdot\frac{3}{n}$, $\eta =\frac{1}{1-k/n}\cdot\frac{1}{n}$ and $\lambda=\Lambda_{2k}$, Corollary~\ref{corollary:identifiability_core} tells us that the means must be close:
    \majoredit
    \begin{align}
        \norm{\hat\mu - \hat\mu'}_{\covone}^2&\le \frac{1}{1-\psdjump} \cdot \lambda\paren{\rho +2\eta}^2 \\
            &\le \frac{1}{1-\psdjump}\cdot  \Lambda_{2k} \cdot \paren{\frac{1}{(1-k/n)^2}\cdot\frac{5}{n}}^2.
    \end{align}
    Finally, recall that $\Lambda_{2k}=(1-\psdjump)^{-2k}\cdot\lambda_0$.
    \editdone
\end{proof}

\subsection{Families of Largest Cores}

Before presenting the proofs of Lemmas~\ref{lemma:weights_are_core} and~\ref{lemma:mean_weights_close}, we establish a few facts about the operation of $\stablemean$, Algorithm~\ref{alg:stable_mean}.

Fix a reference set $\R\subseteq [n]$ and a discretization parameter $k\in \bbN$.
For each $\ell = 0,1,\ldots, 2 k$, $\stablemean$  calls $\mathtt{LargestCore}$ (Algorithm~\ref{alg:largest_core}) to find the largest $\paren{\abs{\R}-\ell, \Lambda_\ell, \hat\Sigma}$-core for $\xx$ with respect to $\R$, where $\hat\Sigma$ is the nonprivate  estimate produced by $\stablecovariance$.
Denote by $\braces{S_{\ell}}_{\ell=0}^{2k}$ this family of largest cores for $\xx$ with respect to $\R$.
In this section we will discuss some properties of families of this form.
As in Section~\ref{sec:covariance} and Lemma~\ref{lemma:family_good_sets}, we show that families of cores on adjacent data sets are tightly interwoven.
\begin{lemma}\label{lemma:family_cores}
    Fix $\R\subseteq[n]$, $\lambda_0>0$, and $k\in \bbN$.
    Let $\xx$ and $\xx'$ be adjacent data sets that differ in index $i^*$.
    Let $\covone$ and $\covtwo$ be positive definite matrices that satisfy $(1-\psdjump)\covone\preceq \covtwo\preceq \frac{1}{1-\psdjump}\covone$ for $\psdjump=\psdjumpdef$.
    Assume $\psdjump<1$.
    
    Let $\braces{S_{\ell}}_{\ell=0}^{2k}$ be the family of sets where $S_\ell$ is the largest $(\abs{\R}-\ell, \Lambda_\ell, \covone)$-core for $\xx$ with respect to $\R$.
    Let $\braces{T_{\ell}}_{\ell=0}^{2k}$ the corresponding family of largest $(\abs{\R}-\ell, \Lambda_\ell, \covtwo)$-cores for $\xx'$ with respect to $\R$.
    The following properties hold:
    \begin{enumerate}
        \item[(i)] For any $0\le \ell\le \ell'\le 2k$, we have $S_{\ell}\subseteq S_{\ell'}$.
        \item[(ii)] For any $0\le \ell<2k$, we have $S_{\ell}\setminus\braces{i^*}\subseteq T_{\ell+1}$.
    \end{enumerate}
\end{lemma}
\begin{proof}
    Statement (i) is direct from our definition of core: decreasing the neighbor threshold and increasing the outlier threshold results in a less restrictive condition, so $S_\ell$ is a $(\abs{\R}-\ell', \Lambda_{\ell'},\covone)$-core for $\xx$ with respect to $\R$.
    Since $S_{\ell'}$ contains all such cores (recall Remark~\ref{remark:unique_cores}), $S_\ell\subseteq S_{\ell'}$.

    When $\covone \succeq (1-\psdjump)\covtwo$, for any vector $v$ we have $\norm{v}_{\covone}^2\le \frac{1}{1-\psdjump}\norm{v}_{\covtwo}^2$.
    Now, $\smini$ is a $(\abs{\R}-\ell, \Lambda_{\ell}, \covone)$-core for $\xx$ with respect to $\R$: for all $i\in \smini$ there exists a set $N_i \subseteq \R$ such that
    \begin{align}
        \abs{N_i}\ge \abs{\R}-\ell 
            \quad\text{and}\quad
        \forall j\in N_i,~\norm{\xx_i - \xx_j}_{\covone}^{2}\le 2 \Lambda_{\ell}.
    \end{align}
    (This condition holds for all $i\in S_\ell$, so it holds for all $i\in \smini$.)
    Changing the covariance, we see that $\smini$ is a $(\abs{R}-\ell, \Lambda_{\ell'}, \covtwo)$-core for $\xx$ with respect to $\R$.

    We now want to argue that $\smini$ is a core for $\xx'$.
    This is not quite trivial, since $i^*$ could be a member of our reference set.
    We must decrease the degree threshold by one: for all $i\in\smini$, there exists a set $N_i'=N_i\setminus\braces{i^*}\subseteq \R$ such that
    \begin{align}
        \abs{N_i'}\ge \abs{\R}-(\ell+1) 
            \quad\text{and}\quad
        \forall j\in N_i',~\norm{\xx_i' - \xx_j'}_{\covone}^{2}\le   2\cdot \Lambda_{\ell+1},
    \end{align}
    where we have used the fact that for all $i$ in $\smini$ or $N_i'= N_i \setminus\braces{i^*}$ we have $\xx_i=\xx_i'$. 
    Thus $\smini$ is a $(\abs{\R}-(\ell+1), \Lambda_{\ell+1}, \covtwo)$-core for $\xx'$ with respect to $\R$, which implies $\smini\subseteq T_{\ell+1}$ as we previously argued.
\end{proof}

\subsection{The Weights Are Good (Proof of Lemma~\ref{lemma:weights_are_core})}

Lemma~\ref{lemma:weights_are_core} says that, when $\stablemean$ on input $\xx$ computes $\score<k$ and the set of reference points is degree-representative, the weights produced are a core for $\xx$.
\begin{proof}[Proof of Lemma~\ref{lemma:weights_are_core}]
    For $\ell\in\braces{0,\ldots,2k}$, let $S_\ell$ denote the largest $(\abs{\R}-\ell, \Lambda_\ell, \hat\Sigma)$-core for $\xx$ with respect to $\R$.
    We know that $S_{2k}$ contains all such cores (recall Remark~\ref{remark:unique_cores}).
    In particular, if an index $i$ is in the support of $w$, then $i\in S_{2k}$.
    This directly implies that $w$ is an $(\abs{R}-2k, \Lambda_{2k}, \hat\Sigma)$-weighted-core for $\xx$ with respect to $\R$.
    (As a technical point, note that $\score<k$ implies that there exists an $\ell^* <k$ with $S_{\ell^*}\neq \emptyset$, so the weights returned are not the zero vector.) 

    Now suppose for contradiction that the weights $w$ are not an $(n/2+1, \Lambda_{2k}, \hat\Sigma)$-weighted-core for $\xx$ with respect to $[n]$.
    Then there is an index $i\in \supp{w}$ such that $z_i\le n/2$ (here $z_i$ and $\tilde{z}_i$ are as in Definition~\ref{def:degree_representative}).
    Since $\R$ is degree-representative for $\xx$, this implies that $\frac{1}{\abs{\R}} \tilde{z}_i \le \frac 2 3$.
    However, since $i$ is in the weighted core, it has many neighbors in $\R$: $\tilde{z}_i$ is at least $\abs{\R} - 2k > \abs{\R} - \frac{2\abs{\R}}{3} = \frac{2\abs{\R}}{3}$.
    Thus we have arrived at a contradiction.
\end{proof}

\subsection{The Weights Are Stable (Proof of Lemma~\ref{lemma:mean_weights_close})}

Lemma~\ref{lemma:mean_weights_close} says that $\stablemean$ produces nearly identical weight vectors on adjacent inputs (assuming the scores are not too large).
\begin{proof}[Proof of Lemma~\ref{lemma:mean_weights_close}]
    For $\ell\in\braces{0,\ldots,2k}$, let $S_\ell$ denote the largest $(\abs{\R}-\ell, \Lambda_\ell, \hat\Sigma)$-core for $\xx$ with respect to $\R$.
    Similarly let $T_\ell$ denote the largest $(\abs{\R}-\ell, \Lambda_\ell, \covtwo)$-core for $\xx'$ with respect to $\R$.
    Recall notation from Algorithm~\ref{alg:stable_mean}: for each $i$ we compute counts $c_i =\sum_{\ell=k+1}^{2k} \indicator{i\in S_{\ell}}$, normalizing constant $Z= \sum_{i\in[n]} c_i$, and weights $w_i = \frac{c_i}{Z}$.  %
    Let $c_i', Z'$, and $w_i'$ denote the same values computed on $\xx'$.
    As in the proof of Lemma~\ref{lemma:covariance_weights_close}, let $\Delta_i=c_i - c_i'$.

    We first show that the normalizing constants computed under $\xx$ and $\xx'$ are similar.
    Suppose $\xx$ and $\xx'$ differ on index $i^*$.
    We know that $\abs{\Delta_{i^*}}\le k$, since the counts are bounded between $0$ and $k$.

    Lemma~\ref{lemma:family_cores} tells us that, for all $i\neq i^*$, $\Delta_i\in \braces{-1,0,+1}$. 
    We also have an upper bound on the number of indices $i\neq i^*$ where $\Delta_i \neq 0$.
    Since we computed $\score < k$ in both cases, the core $S_{k}$ has size at least $n-k+1$.
    Furthermore, for all $\ell>k$ we know that $S_{k}\setminus\braces{i^*}\subseteq S_{\ell}$ and $S_{k}\setminus\braces{i^*}\subseteq T_{\ell}$.
    Thus, for all but at most $k$ indices $i\neq i^*$ (namely, those indices not in $S_{k}\setminus\braces{i^*}$), we have $\Delta_i=0$.
    So $\abs{Z - Z'} = \abs{\sum_i c_i - c_i - \Delta_i} \le 2k$, since the $i^*$ term may have magnitude $k$ and at most $k$ indices have magnitude one. 
    
    The same facts imply that $Z'$ is large:
    \begin{align}
        Z' = \sum_{i=1}^n \sum_{\ell=k+1}^{2k} \indicator{i\in T_\ell} &= \sum_{\ell=k+1}^{2k} \abs{T_{\ell}}\\
            &\ge k \cdot \abs{S_{k}\setminus\braces{i^*}}.
    \end{align}
    This is true for $Z$, too, so we have $Z,Z'\ge k(n-k)$.

    \majoredit
    Counts $c_{i^*}$ and $c_{i^*}'$ are at most $k$ by construction.
    Thus $w_{i^*} = \frac{c_{i^*}}{Z} \le \frac{1}{n-k}=\frac{1}{1-k/n}\cdot \frac 1 n$.
    The same holds for $w_{i^*}'$.

    Finally, we account for the differences between $w$ and $w'$ on the remaining indices.
    We have
    \begin{align}
        \sum_{i\neq i^*} \abs{w_i-w_i'} &= \sum_{i\neq i^*} \abs{\frac{c_i}{Z} - \frac{c_i}{Z'} - \frac{\Delta_i}{Z'}} \\
            &\le \sum_{i\neq i^*} c_i \abs{\frac{1}{Z} - \frac{1}{Z'}} +\abs{\frac{\Delta_i}{Z'}} \\
            &\le n \cdot \abs{\frac{1}{Z} - \frac{1}{Z'}} \cdot \max_{i\neq i^*} \braces{c_i  } + \frac{k}{Z'}, \label{eq:old:diff_normalizers}
    \end{align}
    arriving at $k/Z'$ because at most $k$ indices have $\abs{\Delta_i}=1$, since we excluded $i^*$.
    Because $Z'$ is large, the second term is at most $\frac{1}{1-k/n}\cdot\frac{1}{n}$.
    The first term in Equation~\eqref{eq:old:diff_normalizers} is also small: we have
    \begin{align}
        \abs{\frac{1}{Z} - \frac{1}{Z'}} &= \abs{\frac{Z'}{ZZ'} - \frac{Z}{ZZ'}} \\
            &\le \frac{2k}{k^2 (n-k)^2}.
    \end{align}
    Since $c_i\le k$ we can bound the first term as  
    \begin{align}
        n \cdot \abs{\frac{1}{Z} - \frac{1}{Z'}} \cdot \max_i \braces{c_i} \le \frac{2nk^2}{k^2(n-k)^2} \le \frac{1}{(1-k/n)^2}\cdot \frac{2}{n}.
    \end{align}
    Using $\frac{1}{1-k/n}\le \frac{1}{(1-k/n)^2}$ finishes the proof.
    \editdone
\end{proof}

\subsection{The Score Has Low Sensitivity}\label{sec:mean:sensitivity}

In this subsection we introduce the function $g(\cdot)$ that $\stablemean$ computes to determine the value of $\score$, which it returns along with the mean estimate.
We prove that $\score$ has low sensitivity, which will allow us to use it as an input to $\PTR$, our propose-test-release subroutine.
\begin{definition}
    Fix $\xx$, $\Sigma$, $\lambda_0$, $k$,  and $\R\subseteq [n]$.
    For $\ell\in\braces{0,1,\ldots,k}$, let $S_{\ell}$ denote the largest $(\abs{\R}-\ell, e^{\ell/k}\lambda_0, \Sigma)$-core for $\xx$ with respect to $\R$, with $S_{\ell}=\emptyset$ if no such core exists.
    Define
    \begin{itemize}
         \item $f(\xx,\Sigma) \defeq \min_{\ell \in \braces{0,\ldots,k}} \paren{n - \abs{S_{\ell}} + \ell}$ and 
        \item  $g(\xx,\Sigma) \defeq \min\braces{f(\xx,\Sigma), k}$.
    \end{itemize}
\end{definition}

\newcommand{\coresensitivity}{
    Fix $\lambda_0>0$, $k\in\bbN$, and $\R\subseteq [n]$. 
    Let $\xx$ and $\xx'$ be adjacent data sets.
    Let $\covone$ and $\covtwo$ be positive definite matrices satisfying $(1-\psdjump)\covone \preceq \covtwo \preceq \frac{1}{1-\psdjump}\covone$ for $\psdjump = \frac{1}{1-2k\lambda_{2k}}\cdot \frac{4\lambda_{2k}}{n}$.
    Assume $n>2 k \lambda_{2k}$.
    Then $\abs{g(\xx, \covone) - g(\xx',\covtwo)}\le 2$.
}
\begin{lemma}\label{lemma:core_sensitivity}
    \coresensitivity
\end{lemma}
This proof is nearly identical to that of Lemma~\ref{lemma:good_set_sensitivity}, except we apply Lemma~\ref{lemma:family_cores} instead of Lemma~\ref{lemma:family_good_sets}.
For completeness, Appendix~\ref{app:deferred} contains a restatement (as Claim~\ref{lemma:core_sensitivity_restated}) and proof.

\fi

\ifCOLT
\else
    \section{Private Covariance Estimation and Fast Learning of Gaussians}\label{sec:private_covariance_estimation}
    
As discussed in Section~\ref{sec:techniques}, our privacy analysis argues that on adjacent inputs $x$ and $x'$ we either fail or produce covariance estimates $\covone$ and $\covtwo$ such that $\cN(0,\covone)\approx_{(\eps,\delta)} \cN(0,\covtwo)$.
As~\cite{alabi2022privately} point out, with a stronger stability guarantee one could offer indistinguishability guarantees for multiple independent draws from these distributions.
With enough such samples, one could form an accurate private estimate of the covariance matrix.

We formalize this connection below.
To simplify our presentation, we assume $d = \Omega(\log n)$ and focus on guarantees for Gaussian data that hold with high constant probability.
We state the guarantees for spectral norm; one obtains the guarantees for Frobenius norm immediately via the inequality $\norm{\cdot}_F \le \sqrt{d} \norm{\cdot}_2$.

\begin{theorem}\label{thm:private_covariance_estimation}
    Fix $\eps\in(0,1)$, $\delta\in (0,\eps/1000)$, and $n,d\in \bbN$.
    Assume $d = \Omega(\log n)$.
    Algorithm~\ref{alg:private_covariance_estimation} takes a data set of $n$ points, each in $\bbR^d$, privacy parameters $\eps,\delta$, and an outlier threshold $\lambda_0$.
    \begin{itemize}
        \item For any $\lambda_0\ge 1$, Algorithm~\ref{alg:private_covariance_estimation} is $(\eps,\delta)$-differentially private.
    \item 
    Suppose $\xx$ is drawn i.i.d.\ from $\cN(0, \Sigma)$ where $\Sigma\in\bbR^{d\times d}$ is positive definite. 
    There exists absolute constants $K_1, K_2$, and $K_3$ such that, if $\lambda_0 = K_1 d$ and $n\ge K_2 d \log (1/\delta)/\eps$,
    then with high constant probability Algorithm~\ref{alg:private_covariance_estimation} does not fail and instead returns a positive semidefinite matrix $\tilde\Sigma\in\bbR^{d\times d}$ such that
    \begin{align}
        \norm{\Sigma^{-1/2}\tilde\Sigma\Sigma^{-1/2} - \bbI}_2 \le K_3 \paren{\sqrt{\frac{d}{n}}  + \frac{d^{3/2}\sqrt{\log 1/\delta}}{\eps n} }.
    \end{align}
    
        \item Algorithm~\ref{alg:private_covariance_estimation} can be implemented to require: 
            one product of the form $A^TA$ for $A\in\bbR^{n\times d}$, 
            one product of the form $AB$ for $A\in \bbR^{n\times d}$ and $B\in\bbR^{d\times d}$, 
            one product of the form $A^TA$ for $A\in \bbR^{\N\times d}$, 
            one product of the form $AB$ for $A\in \bbR^{\N\times d}$ and $B\in\bbR^{d\times d}$, 
            one inversion and one eigenvalue decomposition of a positive definite matrix in $\bbR^{d\times d}$ to logarithmic bit complexity, 
            and further computational overhead of $\tilde{O}(nd/\eps)$.
        Here $\N$ is the number of synthetic samples, set in Algorithm~\ref{alg:private_covariance_estimation}.
    \end{itemize}
\end{theorem}

Claim~\ref{claim:gaussian_tv_distance} connects Gaussian parameter estimation (the mean to low  Mahalanobis error, the covariance to low Frobenius error) to learning the distribution itself in low total variation distance.
With this fact, the utility guarantee for Algorithm~\ref{alg:privately_learn_Gaussians} (which simply runs Algorithms~\ref{alg:main} and~\ref{alg:private_covariance_estimation}) is immediate.
The privacy guarantee follows basic composition (Fact~\ref{fact:basic_composition}).

\begin{claim}[\cite{diakonikolas2016robust}]\label{claim:gaussian_tv_distance}
    There exists a constant $K$ such that, for any $\alpha\le \frac 1 2$, vectors $\mu_1,\mu_2\in\bbR^d$, and positive definite $\covone,\covtwo\in\bbR^{d\times d}$, if $\norm{\mu_1-\mu_2}_{\covone}\le \alpha$ and $\norm{\covone^{-1/2}\covtwo\covone^{-1/2}-\bbI}_F\le\alpha$ then $\mathrm{TV}(\cN(\mu_1,\covone),\cN(\mu_2,\covtwo)) \le K\alpha$.    
\end{claim}

\begin{corollary}
    Fix $\eps\in(0,1)$, $\delta\in (0,\eps/10)$, and $n,d\in \bbN$.
    Assume $d = \Omega(\log n)$.
    Algorithm~\ref{alg:privately_learn_Gaussians} takes a data set of $n$ points, each in $\bbR^d$, privacy parameters $\eps,\delta$, and an outlier threshold $\lambda_0$.
    \begin{itemize}
        \item For any $\lambda_0\ge 1$, Algorithm~\ref{alg:privately_learn_Gaussians} is $(2\eps,2\delta)$-differentially private.
        \item Suppose $\xx\simiid\cN(\mu, \Sigma)$ where $\Sigma$ is positive definite. 
            There exists absolute constants $K_1$ and $K_2$ such that, if $\lambda_0 = K_1 d$ and
            \begin{align}
                n\ge K_2 \paren{\frac{d^2}{\alpha^2} + \frac{d^2\sqrt{\log 1/\delta}}{\alpha \eps} + \frac{d\log 1/\delta}{\eps}},
            \end{align}
            then, with high constant probability, Algorithm~\ref{alg:privately_learn_Gaussians} does not fail and instead returns a pair $(\tilde\mu,\tilde\Sigma)$ such that
                $\mathrm{TV}(\cN(\mu,\Sigma), \cN(\tilde\mu,\tilde\Sigma))\le \alpha$.
    \end{itemize}
\end{corollary}

\subsection{Discussion of Private Covariance Estimation}

The topic of differentially private Gaussian covariance estimation has received much attention recently.
Under the assumption that the true covariance satisfies $\bbI\preceq \Sigma\preceq \kappa \bbI$ for some $\kappa\ge 1$, \citet*{kamath_KLSU19} provided efficient algorithms for learning in both the spectral and Frobenius norms.
These algorithms are nearly optimal but require a priori information in the form of $\kappa$.
Their error depends only polylogarithmically on $\kappa$ (in contrast with prior approaches, with error $\Omega(\mathrm{poly}(\kappa))$), but such prior knowledge is not generally available).
In what follows, we focus on learning \textit{unrestricted} covariances, when we do not have such prior knowledge.

With $n\gtrsim d^{3/2}$ samples, Algorithm~\ref{alg:private_covariance_estimation} returns a private covariance estimate that is accurate in spectral norm. 
To the best of our knowledge, this is the first differentially private polynomial-time algorithm achieving such a guarantee for unrestricted Gaussian distributions, closing an open question posed by~\cite{alabi2022privately}.
This dependence on the dimension is known to be optimal in the regime of $\alpha = O(1/\sqrt{d})$ \citep{kamath2022new} and matches the standard Gaussian mechanism applied under the assumption that $\bbI \preceq \Sigma\preceq O(1) \bbI$.

In the past two years, many papers have established techniques for privately learning the covariance of unrestricted Gaussian distributions to $\alpha$ error in rescaled Frobenius norm, also called the Mahalanobis norm for matrices.
Such an estimator, combined with a standard mean estimation procedure, allows one to learn the entire distribution to $O(\alpha)$ total variation distance (see Claim~\ref{claim:gaussian_tv_distance}). 
The optimal sample complexity for this task is
\begin{align}
     n \gtrsim \frac{d^2}{\alpha^2} + \frac{d^2}{\alpha\eps} + \frac{\log 1/\delta}{\eps}. \label{eq:learn_gaussian_complexity}
\end{align}
The information-theoretic upper bound is due to \cite{aden2021sample}; the first term is required for nonprivate estimation, the third is required even for estimating the mean of a univariate Gaussian with known variance and unrestricted mean \citep{karwa2017finite}, and the second was recently proved to be tight \citep{kamath2022new}.
The upper bound of~\cite{aden2021sample} is non-constructive; later work by~\cite{liu2022differential} gave an exponential-time algorithm with nearly the same guarantees.
A group of concurrent and independent papers soon gave polynomial-time algorithms \citep{kamath2021private, ashtiani2021private, kothari2021private}, with \cite{ashtiani2021private} achieving the optimal $d^2$ dimension-dependence. 
(Using tools from~\cite{ashtiani2021private}, the framework of \cite{tsfadia2022friendlycore} achieves a similar error bound.)
Very recently, simultaneous work by \cite{hopkins2022robustness} and \cite{alabi2022privately} gave polynomial-time and sample-optimal estimators that also satisfy robustness against adversarial perturbation to the input.
The former's guarantees for this task are also the best-known in polynomial time, matching Equation~\ref{eq:learn_gaussian_complexity} up to logarithmic factors in $d, \alpha, \eps$, and $\log 1/\delta$.

Our work (combining Algorithms~\ref{alg:main} and~\ref{alg:private_covariance_estimation}) matches the optimal sample complexity up to logarithmic factors and runs in time $\tilde{O}(nd^{\omega-1} + nd/\eps)$.
To the best of our knowledge, the fastest known algorithm takes time $\tilde{O}(nd^{\omega-1}+d^{\omega+1}/\eps)$ and arises by combining ideas of \cite{ashtiani2021private} and \cite{tsfadia2022friendlycore}.
For this task's typical parameter setting of $n\gtrsim d^2/\eps$, both of these expressions are asymptotically dominated by their first terms, corresponding to the time needed to compute the covariance of the entire input data.

We remark that the exact computation of eigenvalue decompositions is not possible; suitable approximation algorithms are well-understood \citep[see, e.g.,][]{trefethen1997numerical}.

\begin{algorithm2e}\label{alg:private_covariance_estimation}
    \SetAlgoLined
    \SetKwInOut{Input}{input}
    \SetKwInOut{Require}{require}
    
    \Input{data set $\xx\in \bbR^{n\times d}$, privacy parameters $\eps,\delta$, outlier threshold $\lambda_0$}
    \Require{$n\ge 10k \lambda_0$}
    \BlankLine
    $k\gets \ceiling{\frac{4\log 2/\delta}{\eps}}+4$\;
    $\N\gets \floor{\frac{1}{24}\cdot \frac{n^2\eps^2}{\lambda_0^2\log 2/\delta}}$\;
    $\hat\Sigma, \score \gets \stablecovariance(\xx,\lambda_0,k)$\tcc*[r]{$\hat\Sigma\in\bbR^{d\times d}, \score\in \bbN$}
    \eIf{$\cM_{\mathrm{PTR}}^{\eps/2,\delta/2}\paren{\score} = \PASS$}{
        Draw $Z_1,\ldots,Z_{\N}\overset{iid}{\sim}\cN(0,\hat\Sigma)$\;
        \KwRet $\frac{1}{\N}\sum_{i=1}^{\N} Z_iZ_i^T$\;
    }{
        \KwRet \FAIL
    }
    \caption{Private Covariance Estimation, $\cA_{\mathrm{cov}}^{\eps, \delta, \lambda_0}(x)$}
\end{algorithm2e}

\subsection{Proof of Theorem~\ref{thm:private_covariance_estimation}}

We use the following lemma, which can be interpreted as an analog of Claim~\ref{claim:covariance_indistinguishability} (used in BGSUZ) that uses advanced composition, requiring roughly a factor of $\sqrt{N}$ more samples to repeat the process $\N$ times.
\begin{claim}[\cite{alabi2022privately}, Theorem 5.1]\label{claim:samples_from_covariance}
    Fix $\eps,\delta\in (0,1)$.
    Let $\covone, \covtwo$ be positive definite matrices satisfying 
    \begin{align}
        \norm{\covtwo^{-1/2}\covone\covtwo^{-1/2}-\bbI}_{F},  \norm{\covone^{-1/2}\covtwo\covone^{-1/2}-\bbI}_{F} = \Delta
    \end{align}
    for any $\Delta$ satisfying $\Delta \le \frac{\eps}{8\log 1/\delta}<1$.
    For any $\N\le \frac{\eps^2}{8\Delta^2\log 1/\delta}$, we have 
        $\cN(0,\covone)^{\otimes \N}\approx_{(\eps,\delta)} \cN(0,\covtwo)^{\otimes \N}$,
    where $p^{\otimes \N}$ denotes the $\N$-fold product distribution, i.e., $\N$ independent copies of $p$.
\end{claim}

\begin{proof}[Proof of Theorem~\ref{thm:private_covariance_estimation}]
    As in the proof of Lemma~\ref{lemma:privacy}, requiring $n\ge 10 k\lambda_0$ suffices to ensure that drawing $N=1$ samples preserves privacy. 
    (Note that we save a factor of 2 because Algorithm~\ref{alg:private_covariance_estimation} does not pair samples.)

    Set $\eps' = \eps/2$ and $\delta'=\delta/2$.
    Recall from Lemma~\ref{lemma:covariances_close} that, when the scores are below $k$, we have $\norm{\covone^{-1/2}\covtwo\covone^{-1/2}-\bbI}_{\tr} \le \frac{25}{3}\cdot\frac{1}{1-\psdjump}\cdot\frac{\lambda_0}{n}$ for $\gamma = \psdjumpdef$ (and the same inequality with $\covone$ and $\covtwo$ swapped).
    To apply Claim~\ref{claim:samples_from_covariance}, we use the fact that $\norm{\cdot}_F\le \norm{\cdot}_{\tr}$.
    Our assumptions on $n, \eps,$ and $\delta$ mean that $k\ge 34$ and thus $\gamma\le 0.017$, so we have a trace norm bound of $\Delta \le \frac{25}{3}\cdot \frac{1}{1-\psdjump}\cdot\frac{\lambda_0}{n}\le \frac{17}{20} \cdot \frac{\lambda_0}{n}$.
    Claim~\ref{claim:samples_from_covariance} says we can draw up to
    \begin{align}
        N &= \left\lfloor \frac{1}{\Delta^2}\cdot  \frac{(\eps')^2}{8 \log 1/\delta'} \right\rfloor \\
            &= \left\lfloor \frac{400 n^2}{289 \lambda_0^2} \cdot\frac{\eps^2}{32 \log 2/\delta} \right\rfloor \\
            &= \left\lfloor \frac{25}{578} \cdot\frac{n^2}{\lambda_0^2} \cdot \frac{\eps^2}{\log 2/\delta} \right\rfloor
    \end{align}
    samples while preserving privacy.
    (In Algorithm~\ref{alg:private_covariance_estimation} we replace $25/278$ with the slightly smaller $1/24$.)
    The remainder of the privacy argument follows from the guarantees of propose-test-release (Lemma~\ref{lemma:main_sensivity} for sensitivity of the score function and Claim~\ref{claim:PTR} for privacy of $\PTR$) and basic composition.

    We analyzed the running time of $\stablecovariance$ in Section~\ref{sec:running_time}.
    To generate the private covariance estimate, we first we compute the eigenvalue decomposition of $\hat\Sigma$ and from it produce $\hat\Sigma^{1/2}$.
    We then draw $\N$ samples from $\cN(0,\bbI)$ (in time $\tilde{O}(\N d)$), rescale them by $\hat\Sigma^{1/2}$, and compute the empirical covariance of the rescaled points.

    For accuracy, first note that we have assumed $n\gtrsim d\log (1/\delta)/\eps$ so that the input requirement on $n$ is satisfied and we do not fail immediately.  
    Beyond that, the accuracy argument uses two straightforward applications of Claim~\ref{claim:concentration_of_covariance}.
    Let $\Sigma_\yy$ be the empirical covariance of $\yy$ and $\tilde\Sigma$ the private estimate we release.
    We add and subtract $\Sigma_\yy$ and apply the triangle inequality:
    \begin{align}
        \norm{\Sigma^{-1/2} \tilde\Sigma\Sigma^{-1/2} - \bbI}_2 
            &= \norm{\Sigma^{-1/2}\paren{\tilde\Sigma - \Sigma}\Sigma^{-1/2}}_2 \\
            &= \norm{\Sigma^{-1/2}\paren{\tilde\Sigma -\Sigma_{\yy} + \Sigma_{\yy} - \Sigma}\Sigma^{-1/2}}_2 \\
            &\le  \norm{\Sigma^{-1/2}\paren{\tilde\Sigma -\Sigma_{\yy}}\Sigma^{-1/2}}_2 
                + \norm{\Sigma^{-1/2}\paren{\Sigma_{\yy} - \Sigma}\Sigma^{-1/2}}_2 \\
            &= \norm{\Sigma^{-1/2}\paren{\tilde\Sigma -\Sigma_{\yy}}\Sigma^{-1/2}}_2 
                + \norm{\Sigma^{-1/2} \Sigma_{\yy}\Sigma^{-1/2} - \bbI}_2.
    \end{align}
    Since $n\gtrsim d+\log 1/\beta$, with high probability $\Sigma_\yy$ is invertible and $\norm{\Sigma^{1/2}\Sigma_{\yy}^{-1/2}}_2 = 1 + o(1)$, so applying Cauchy-Schwarz we have
    \begin{align}
        \norm{\Sigma^{-1/2}\paren{\tilde\Sigma -\Sigma_{\yy}}\Sigma^{-1/2}}_2 
            &= \norm{\Sigma^{-1/2}\paren{\Sigma_{\yy}^{1/2}\Sigma_{\yy}^{-1/2}}\paren{\tilde\Sigma -\Sigma_{\yy}}\paren{\Sigma_{\yy}^{-1/2}\Sigma_{\yy}^{1/2}}\Sigma^{-1/2}}_2 \\
            &\lesssim \norm{\Sigma_{\yy}^{-1/2}\paren{\tilde\Sigma -\Sigma_{\yy}}\Sigma_{\yy}^{-1/2}}_2 \\
            &= \norm{\Sigma_{\yy}^{-1/2}\tilde\Sigma\Sigma_{\yy}^{-1/2} - \bbI}_2.
    \end{align}
    Thus we have to control two terms:
    \begin{align}   
        \norm{\Sigma^{-1/2} \tilde\Sigma\Sigma^{-1/2} - \bbI}_2 
            &\lesssim \norm{\Sigma^{-1/2} \Sigma_{\yy}\Sigma^{-1/2} - \bbI}_2 + \norm{\Sigma_{\yy}^{-1/2}\tilde\Sigma\Sigma_{\yy}^{-1/2} - \bbI}_2.
            \label{eq:covariance_triangle_inequality}
    \end{align}

    Since $\Sigma_\yy$ is an empirical estimate of $\Sigma$ with $n$ samples, with high probability the first term in Equation~\eqref{eq:covariance_triangle_inequality} will be at most $\sqrt{d/n}$.

    Similarly, $\tilde\Sigma$ is an empirical estimate of $\Sigma_\yy$ with $\N$ samples, so with high probability the second term in Equation~\ref{eq:covariance_triangle_inequality} will be at most $\sqrt{d/N}$.
    We have $\N \approx \frac{n^2\eps^2}{d^2\log 1/\delta}$, and thus rearranging get an upper bound of $\frac{d^{3/2}\sqrt{\log 1/\delta}}{\eps n}$, as promised.
\end{proof}

\begin{algorithm2e}\label{alg:privately_learn_Gaussians}
    \SetAlgoLined
    \SetKwInOut{Input}{input}
    \SetKwInOut{Require}{require}
    
    \Input{data set $\xx\in \bbR^{n\times d}$, privacy parameters $\eps,\delta$, outlier threshold $\lambda_0$}
    \Require{$n\ge \frac{272 e^2 \lambda_0 \log 2/\delta}{\eps}$}
    \BlankLine

    $\tilde\mu \gets \cA_{\mathrm{main}}^{\eps, \delta, \lambda_0}(x)$\tcc*[r]{Algorithm~\ref{alg:main}}

    \BlankLine
    $\forall i\in [\floor{n/2}], y_i \gets \frac{1}{\sqrt{2}}(x_i - x_{i+\floor{n/2}})$\tcc*[r]{Pair and rescale}
    $\tilde\Sigma \gets \cA_{\mathrm{cov}}^{\eps, \delta, \lambda_0}(y)$\tcc*[r]{Algorithm~\ref{alg:private_covariance_estimation}}
    
    \eIf{$\tilde\mu = \FAIL$ or $\tilde\Sigma = \FAIL$}{
        \KwRet $\FAIL$\;
    }{
        \KwRet $\tilde\mu,\tilde\Sigma$\;
    }
    \caption{Private Learner for Unrestricted Gaussians}
\end{algorithm2e}

\fi

\bibliography{bibliography}

\appendix

\section{Preliminaries}\label{app:preliminaries}

\subsection{Notation and Elementary Facts}\label{sec:organization_and_notation}

Unless stated otherwise, ``$\gtrsim$'' and ``$\lesssim$'' hide absolute constants.
We interpret data set $x=(x_1,\ldots,x_n)$ of points $x_i\in\bbR^d$ as both an ordered $n$-tuple and a matrix $x\in\bbR^{n\times d}$.
Vectors are columns.
Logarithms are base $e$.
We use $[n]$ for the set $\braces{1,\ldots,n}$ and $\bbN =\braces{1,2,\ldots}$ for the natural numbers.
For a vector $v\in \bbR^d$, its support is $\supp{v} = \braces{i\in[d] : v_i\neq 0}$.
The indicator $\indicator{P}\in\braces{0,1}$ equals one when predicate $P$ is true and zero otherwise.
With zero-mean data, we will often use ``covariance'' to refer to the second moment matrix: $\Sigma_x = \frac 1 n \sum_{i=1}^n x_i x_i^T = \frac{1}{n} x^T x$.

Throughout, we use the facts that $1+x\le e^x$ and, for $x\in [0,p]$ with $p<1$, we have $\frac{1}{1-x}\le 1 + \frac{x}{1-p}$. 
In particular, when $p=\frac 1 2$ we have $\frac{1}{1-x}\le 1 + 2x$.

We say that a symmetric matrix $A=\bbR^{d \times d}$ is \emph{positive semidefinite} if $v^T A v\ge 0$ for all $v\in \bbR^{d}$.
If the same statement holds with a strict inequality, we say that $A$ is \emph{positive definite}.
For positive semidefinite matrices $A$ and $B$, we denote the positive semidefinite order (also called the Loewner order) in the standard way: $A\succeq B$ if and only if $A - B$ is positive semidefinite.
This notation extends to $\preceq$, $\succ$, and $\prec$.
Matrix inversion respects the positive semidefinite order (\cite{horn2012matrix}, Corollary 7.7.4.a): if $A,B\succ 0$ then $A\succeq B$ if and only if $A^{-1} \preceq B^{-1}$. 
This allows us to relate Mahalanobis distance and the positive semidefinite order: if matrices $\Sigma_1$ and $\Sigma_2$ satisfy $(1-\gamma)\Sigma_1 \preceq \Sigma_2\preceq \frac{1}{1-\gamma}\Sigma_1$, then for any vector $v$ we have $(1-\gamma)\norm{v}_{\Sigma_1}^2 \le \norm{v}_{\Sigma_2}^2 \le \frac{1}{1-\gamma}\norm{v}_{\Sigma_1}^2$. 

For a  matrix $A\in\bbR^{d\times d}$ with singular values $\sigma_1\ge \cdots\ge\sigma_d$, we use the \emph{spectral norm} $\norm{A}_2 = \sigma_1$, the \emph{Frobenius norm} $\norm{A}_F = \sqrt{\sum_{i=1}^d \sigma_i^2}$, and the \emph{trace norm} $\norm{A}_{\tr}=\sum_{i=1}^d \sigma_i$.
We have $\norm{A}_2 \le \norm{A}_F\le \norm{A}_{\tr}$.

\subsection{Subgaussian Random Variables and Concentration Inequalities}\label{app:subgaussian_facts}

For further discussion on subgaussian random variables, see the textbook by~\cite{vershynin2018high}.

\begin{definition}[Subgaussian Norm]
    Let $\yy\in\bbR $ be a random variable. 
    The \emph{subgaussian norm of $\yy$}, denoted $\norm{\yy}_{\psi_2}$, is 
        $\norm{\yy}_{\psi_2} = \inf \braces{t>0 : \bbE \exp\paren{\yy^2/t^2}\le 2}$.
\end{definition}

\begin{definition}[Subgaussian Random Variable]\label{def:subgaussian_vector}
    Let $\yy\in\bbR^d$ be a random variable with mean $\mu$ and covariance $\Sigma$. 
    Call $\yy$ \emph{subgaussian with parameter $K$} if there exists $K\ge 1$ such that for all $v\in\bbR^d$ we have
    \begin{align}
        \norm{\ip{\yy-\mu, v}}_{\psi_2} \le K \sqrt{ v^T\Sigma v }.
    \end{align}
\end{definition}

The Gaussian distribution $\cN(\mu,\Sigma)$ is subgaussian with parameter $K=O(1)$.

\begin{claim}[Concentration of Norm]\label{claim:concentration_of_norm}
    Let $\yy_1,\ldots,\yy_n$ be drawn i.i.d.\ from a $d$-dimensional subgaussian distribution with parameter $K_\yy$, mean $\mu$, and (full-rank) covariance $\Sigma$.
    There exists a constant $K_1$ such that, with probability at least $1-\beta$, we have both
    \begin{align}
        \norm{\yy_1 - \mu}_{\Sigma}^2 \le K_1 K_{\yy}^2 \paren{d + \log 1/\beta}
  \quad\text{and}\quad 
        \norm{\frac 1 n \sum_{i=1}^n \yy_i - \mu}_{\Sigma}^2 \le K_1 K_{\yy}^2\cdot \frac{d + \log 1/\beta}{n}.
    \end{align}
\end{claim}

\begin{claim}[Concentration of Covariance]\label{claim:concentration_of_covariance}
    Let $\yy_1,\ldots, \yy_n$ be drawn i.i.d.\ from a $d$-dimensional subgaussian distribution with parameter $K_\yy$, mean $\mu=0$, and (full-rank) covariance $\Sigma$.
    Let $\hat\Sigma = \frac{1}{n}\sum_{i=1}^n \yy_i\yy_i^T$ be the empirical covariance.
    There exists absolute constants $K_1$ and $K_2$ such that, for any $\beta\in (0,1)$, if $n\ge K_2\paren{d + \log 1/\beta}$, then
    with probability at least $1-\beta$ we have
    \begin{align}
        \norm{\Sigma^{-1/2}\hat\Sigma\Sigma^{-1/2} - \bbI}_2 \le K_1 K_{\yy}^2\sqrt{\frac{d + \log 1/\beta}{n}}.
    \end{align}
\end{claim}

Finally, we use the following tail bound for hypergeometric distributions to argue that the random set $\R$ of reference points (selected in Algorithm~\ref{alg:main} and provided to $\stablemean$) is degree-representative with high probability. 
\begin{claim}[See~\cite{skala2013hypergeometric}]\label{claim:hypergeometric_tail}
    Suppose an urn contains $N$ balls and exactly $k$ of them are black.
    Let random variable $\yy$ be the number of black balls selected when drawing $n$ balls uniformly from the urn without replacement.
    Then for all $t\ge 0$ we have $\Pr\brackets{\abs{\yy- kn/N} \ge tn}\le 2e^{-2t^2 n}$.
\end{claim}

\subsection{Differential Privacy}\label{app:dp_facts}

For these facts and further background on differential privacy, see the monograph by~\cite{vadhan2017complexity}.

\begin{fact}[Basic Composition]\label{fact:basic_composition}
    For all $1\le i\le K$, suppose mechanism $\cM_i$ is $(\eps,\delta)$-differentially private.
    Then $\paren{\cM_1,\cM_2,\ldots,\cM_K}$ is $(K\eps,K\delta)$-differentially private.
    Moreover, this holds even when the mechanisms are chosen adaptively.
\end{fact}

\begin{fact}\label{fact:group_privacy}
    Suppose for some $\eps$ and $\delta$ that distributions $p_1, p_2,$ and $p_3$ satisfy $p_1 \approx_{(\eps,\delta)} p_2$ and $p_2\approx_{(\eps,\delta)} p_3$.
    Then $p_1 \approx_{(2\eps, (1+e^{\eps})\delta)} p_3$.
\end{fact}

\section{Deferred Proofs: Identifiability and Sensitivity}\label{app:deferred}

In this appendix we prove three statements whose proofs appeared, with only minor changes from the versions we present, in BGSUZ.

We also prove Lemma~\ref{lemma:core_sensitivity} (restated as Lemma~\ref{lemma:core_sensitivity_restated}), whose proof is almost identical to that of Lemma~\ref{lemma:good_set_sensitivity}, and Claim~\ref{claim:PTR} (restated as Claim~\ref{claim:PTR_restated}), the guarantees for our propose-test-release variant.

\begin{lemma}[Restatement of Lemma~\ref{lemma:identifiability_good_set}]\label{lemma:identifiability_good_set_restated}
    \IdentifiabilityGoodSet
\end{lemma}
\begin{proof}
    Assume without loss of generality that $y$ and $y'$ differ on index $i^*=1$ and let $z_1\defeq y_1'$.
    Then we have 
    \begin{align}
        \Sigma_{v} = \sum_{i\in [\m]} v_i \cdot  \paren{\yy_i'} \paren{\yy_i'}^T 
            &= \Sigma_w + v_1 \cdot z_1 z_1^T - w_1 \cdot y_1 y_1^T + \sum_{\substack{i\in [\m] \\ i >1}} (v_i - w_i) \cdot \yy_i \yy_i^T  \label{eq:diff_of_covariances} \\
            &\succeq \Sigma_w + 0 - w_1 \cdot y_1 y_1^T + \sum_{\substack{i\in \mathrm{supp}(w) \\ i>1}} (v_i - w_i) \cdot \yy_i \yy_i^T.
    \end{align}
    The last line follows from the fact that, by restricting to the support of $w$, we only drop positive semidefinite terms from the sum (when $w_i=0$, we have $v_i - w_i\ge 0$).
    We want to lower bound this with $(1-\psdjump)\Sigma_w$, so it remains to show that 
    $-w_1 \cdot y_1 y_1^T + \sum_{i \in \mathrm{supp}(w)} (v_i - w_i)\yy_i\yy_i^T \succeq -\psdjump \Sigma_w$, or equivalently that
    $w_1 \cdot y_1 y_1^T + \sum_{i \in \mathrm{supp}(w)} (w_i - v_i) \yy_i \yy_i^T \preceq \psdjump \Sigma_w$.
    This is an upper bound on the operator norm;
    to prove it we conjugate by $\Sigma_w^{-1/2}$ and apply the triangle inequality plus the assumption that $w$ is $\lambda$-good.
    (Also recall that $\norm{uu^T}_2 = u^T u$ for any vector $u$.)
    \begin{align}
        \biggl\lVert w_1 \cdot \Sigma_w^{-1/2} y_1 y_1^T &\Sigma_w^{-1/2} +\sum_{\substack{i\in \mathrm{supp}(w) \\ i>1}} (w_i - v_i)\cdot \Sigma_w^{-1/2} \yy_i \yy_i^T \Sigma_w^{-1/2} \biggr\rVert_2\\
            &\le \abs{w_1} \cdot \norm{\Sigma_w^{-1/2} y_1 y_1^T \Sigma_w^{-1/2}}_2 + \sum_{\substack{i\in \mathrm{supp}(w) \\ i>1}} \abs{w_i - v_i}\cdot \norm{\Sigma_w^{-1/2} \yy_i \yy_i^T \Sigma_w^{-1/2}}_2 \\
            &=  \abs{w_1} \cdot \norm{\Sigma_w^{-1/2} y_1 }_2^2 + \sum_{\substack{i\in \mathrm{supp}(w) \\ i>1}} \abs{w_i - v_i}\cdot \norm{\Sigma_w^{-1/2} \yy_i }_2^2 \\
            &\le \abs{w_1} \cdot \lambda + \sum_{\substack{i\in \mathrm{supp}(w) \\ i>1}} \abs{w_i - v_i}\cdot \lambda.
    \end{align}
    This is at most $\lambda\paren{\eta + \rho}$ by assumption.
    A symmetrical argument shows that $(1-\psdjump)\Sigma_{v}\preceq \Sigma_w$, which gives us our upper bound on $\Sigma_{v}$.

    A straightforward argument establishes the trace norm inequalities.
    Continuing from Equation~\eqref{eq:diff_of_covariances} and conjugating by $\Sigma_w^{-1/2}$, we have
    \begin{align}
        \Sigma_w^{-1/2} \Sigma_{v} \Sigma_w^{-1/2} - \bbI
            &=  v_1 \cdot\Sigma_w^{-1/2} z_1 z_1^T \Sigma_w^{-1/2} - w_1 \cdot\Sigma_w^{-1/2} y_1 y_1^T \Sigma_w^{-1/2} \\
            &\quad + \sum_{i=2}^{\m} (v_i - w_i)\cdot \Sigma_w^{-1/2} \yy_i \yy_i^T \Sigma_w^{-1/2}.
    \end{align}
    As before, we apply the triangle inequality and $\norm{uu^T}_{\tr}=\norm{u}_2^2$:
    \begin{align}
        \norm{\Sigma_w^{-1/2} \Sigma_{v} \Sigma_w^{-1/2} - \bbI }_{\tr}
            &\le \abs{v_1} \cdot \norm{\Sigma_w^{-1/2} z_1}_2^2 + \abs{w_1}\cdot  \norm{\Sigma_w^{-1/2} y_1}_2^2 \\
            &\quad +\sum_{i=2}^\m \abs{v_i - w_i}\cdot \norm{\Sigma_w^{-1/2} \yy_i }_{2}^2.
    \end{align}
    We cannot uniformly upper bound these squared norms by $\lambda$, since our assumption of $w$'s goodness only applies to points in the support of $w$.
    However, by the positive semidefinite inequalities established before,
    for all $i$ in the support of $v$ we have $\norm{\Sigma_w^{-1/2} \yy_i' }_{2}^2\le \frac{1}{1-\psdjump}\lambda$.

    This argument, combined with $\sum_{i>1}\abs{v_i-w_i}\le \rho$, yields
    \begin{align}
        \norm{\Sigma_w^{-1/2} \Sigma_{v} \Sigma_w^{-1/2} - \bbI }_{\tr} 
            &\le 1.05 \cdot \paren{\abs{v_1} + \abs{w_1} + \rho}\cdot  \lambda \\
            &= 1.05\cdot \gamma_2.
    \end{align}
    The second trace norm inequality is symmetrical.

    In the setting where data sets $y$ and $y'$ are identical, the same proof structure applies directly.
\end{proof}

Observe that, when $\R=[n]$ and the degree threshold $\tau$ is at least $\frac{n}{2}+1$, any two points in $\supp{w}$ must have a ``neighbor'' in common, and thus cannot themselves be too far apart.
We have the following claim, which says that weighted cores which are close in $\ell_1$ norm (equivalently, total variation distance) have close means.
Furthermore, this holds even when the cores are computed under covariances that differ slightly.
\begin{lemma}\label{lemma:identifibility_cores}
    Let $\xx = \xx_1,\ldots, \xx_n$ be a data set. 
    Let $\lambda>0$, $\tau\in\bbN$, and $\R\subseteq [n]$.
    Let $\covone$ and $\covtwo$ be positive definite matrices satisfying $(1-\gamma)\covone \preceq \covtwo \preceq \frac{1}{1-\gamma}\covone$ for $\gamma<1$.
    Let $w$ be a $(\tau, \lambda, \covone)$-weighted-core for $\xx$ with respect to $\R$ and let $v$ be a $(\tau, \lambda, \covtwo)$-weighted-core for $\xx$ with respect to $\R$.
    Assume $\tau>\frac n 2$.
    If $\norm{w - v}_1 \le \rho$, then 
    \begin{align}
        \norm{\mu_w - \mu_v}_{\covone}^2 = \norm{\sum_{i\in[n]}w_i\xx -\sum_{i\in[n]}v_i \xx_i }_{\covone}^2 \le \frac{4 \rho^2 \lambda}{1-\gamma}.
    \end{align}
\end{lemma}
\begin{proof}
    For any $i, j \in \supp{w}~\cup~\supp{v}$, by definition there exist sets of ``nearby'' data points $N_i, N_j\subseteq \R$ of size at least $\tau$. 
    As $\tau>\frac n 2$, we know that $N_i$ and $N_j$ have a nonempty intersection.
    Thus there is a point $\iota \in \R$ such that $\norm{\xx_i - \xx_\iota}_{\covone}^2 \le 4 \lambda$ and $\norm{\xx_{\iota} - \xx_j}_{\covtwo}^2\le 4\lambda$.
    By the assumptions on $\covone,\covtwo$, and $\gamma$, the second inequality also implies $\norm{\xx_{\iota} - \xx_j}_{\covone}^2\le  4 \lambda / (1-\gamma)$.
    Now, working with the (unsquared) norm, we add and subtract $\xx_\iota$ and apply the triangle inequality.
    For any $i, j\in \supp{w}\cup \supp{v}$,
    \begin{align}
        \norm{\xx_i - \xx_j}_{\covone} &= \norm{\xx_i -\xx_{\iota} + \xx_\iota- \xx_j}_{\covone}\\
            &\le \norm{\xx_i - \xx_\iota}_{\covone} + \norm{\xx_\iota - \xx_j}_{\covone}\\
            &\le 4\sqrt{ \lambda /(1-\gamma)}.
    \end{align}
    
    Let $C$ be a multivariate random variable obeying distribution $w$, and let random variable $D$ obey $v$.
    Thus $\norm{\mu_w - \mu_v}_{\covone}$ is just $\norm{\bbE[C-D]}_{\covone}$.
    The distributions $w$ and $v$ have total variation distance at most $\rho/2$ (since total variation distance is exactly half the $\ell_1$ distance), so there exists a maximum coupling of $w$ and $v$ and random variables $C'$, $D'$ such that $\bbE[C-D] = \bbE[C'-D']$ and $\Pr[C'\neq D']\le \rho/2$.
    We apply Jensen's inequality to the norm and drop the part of the expectation corresponding to $C'=D'$, because it is zero:
    \begin{align}
         \norm{\mu_w - \mu_v}_{\covone} &= \norm{\bbE[C'-D']}_{\covone} \\
            &\le \bbE\brackets{\norm{C'-D'}_{\covone}} \\
            &= \Pr[C' \neq D']\cdot \bbE\brackets{\norm{C'-D'}_{\covone}\mid C'\neq D'}.
    \end{align}
    With an upper bound on $\Pr[C'\neq D]$ and a uniform upper bound on the values inside the expectation, we arrive at
    \begin{align}
        \norm{\mu_w - \mu_v}_{\covone} &\le \frac{\rho}{2} \cdot 4\sqrt{\lambda/(1-\gamma)} = 2\rho \sqrt{ \lambda/(1-\gamma)}.
    \end{align}
    Squaring finishes the proof.
\end{proof}

Corollary~\ref{corollary:identifiability_core} shows that the same argument applies, with little modification, to the setting where the cores are computed on adjacent data sets.
Note that we make a slightly stronger assumption, asking that $\tau>\frac n 2 + 1$ instead of $\frac n 2$.
Additionally, we separate the $\ell_1$ distance on shared points from the weights on points that differ.
\begin{corollary}[Restatement of Corollary~\ref{corollary:identifiability_core}]\label{corollary:identifiability_core_restated}
    \IdenfitiabilityCore
\end{corollary}
\begin{proof}[Proof of Corollary~\ref{corollary:identifiability_core}]
    From $w$ and $v$ we construct cores $w^+$ and $v^+$ over $\xx\cup \xx'$ and apply Lemma~\ref{lemma:identifibility_cores}.
    
    Let $i^*\in [n]$ be the index in which $\xx$ and $\xx'$ differ.
    Our data sets are ordered and not strictly sets, so interpret $\xx\cup \xx' = (\xx_1, \xx_2,\ldots,\xx_n, \xx_{i^*}')$ as a data set of size $n+1$.
    Let $\R^+ \gets \R \cup \braces{n+1}$ be the similarly extended reference set.
    Already, $w^+ = (w_1, \ldots,w_n, 0)$ is an $(m, \lambda, \covone)$-core for $\xx\cup \xx'$ with respect to $\R^+$.
    Similarly,
    \begin{align}
        v^+ = (v_1, \ldots, v_{i-1}, 0, v_{i+1},\ldots, v_n, v_i)
    \end{align}
    is an $(\tau, \lambda, \covtwo)$-core for $\xx\cup \xx'$ with respect to $\R^+$.
    Lemma~\ref{lemma:identifibility_cores} asks that $\tau > \frac{n+1}{2}$, which is satisfied: we have assumed that $\tau > \frac n 2 + 1$.
    So it remains to calculate $\norm{w^+ - v^+}_1$, which is simple:
    \begin{align}
        \norm{w^+ - v^+}_1 &= \sum_{j\in[n+1]} \abs{w_j^+ - v_j^+}\\
            &= \abs{w_i^+ - v_i^+} + \abs{w_{n+1}^+ - v_{n+1}^+} + \sum_{j\in[n+1]\setminus{\braces{i,n+1}}} \abs{w_j^+ - v_j^+} \\
            &= \abs{w_i^+ - 0} + \abs{0 - v_{n+1}^+} + \sum_{j\in[n+1]\setminus{\braces{i,n+1}}} \abs{w_j^+ - v_j^+} \\
            & = \abs{w_{i^*}} + \abs{v_{i^*}} + \sum_{i\neq i^*} \abs{w_i - v_i}.
    \end{align}
    By assumption, this is at most $\rho + 2\eta$.
\end{proof}

\begin{lemma}[Restatement of Lemma~\ref{lemma:core_sensitivity}]\label{lemma:core_sensitivity_restated}
    \coresensitivity
\end{lemma}
\begin{proof}
    Without loss of generality, assume $g(\xx,\covone) \le g(\xx',\covtwo)$.
    We will show that $g(\xx',\covtwo) \le g(\xx,\covone) + 2$ by analyzing two cases.

    \textbf{Case 1:} Suppose $g(\xx, \covone)=k$.
        Then $g(\xx', \covtwo) \le g(\xx,\covone) \le g(\xx,\covone)+2$ by construction, since $g(\cdot)$ is capped at $k$.

    \textbf{Case 2:} Suppose $g(\xx,\covone)<k$.
        This can only happen when $g(\xx,\covone)<k$, so there exists an $\ell^*\in\braces{0,\ldots,k}$ and subset $S_{\ell^*}$ that (i) is a $(\abs{\R}-\ell, \Lambda_{\ell^*}, \covone)$-core for $\xx$ with respect to $\R$ and (ii) satisfies $n-\abs{S_{\ell^*}} + \ell^* < k$.
        Furthermore, we know that $\ell^*\neq k$, since $n-\abs{S_{\ell^*}}$ is nonnegative.
        
        We can now apply Lemma~\ref{lemma:family_cores}.
        For $\ell\in\braces{0,\ldots,k}$, let $T_{\ell}$ be the largest $(\abs{\R}-\ell, \Lambda_{\ell}, \covtwo)$-core for $\xx'$ with respect to $\R$.
        Lemma~\ref{lemma:family_cores} says that $S_{\ell^*}\setminus{i^*} \subseteq T_{\ell^*+1}$.
        This allows us to upper bound $g(\xx',\covtwo)$:
        \begin{align} 
        g(\xx',\covtwo) \le g(\xx',\covtwo) &= \min_{\ell\in\braces{0,\ldots,k}} n - \abs{T_{\ell}} +\ell \\
            &\le n - \abs{T_{\ell^*+1}} + \ell^* + 1\\
            &\le n - \paren{|S_{\ell^*}| - 1} + \ell^* + 1\\
            &= g(\xx, \covone) + 2.
    \end{align}
    So we are done.
\end{proof}

\begin{claim}[Restatement of Claim~\ref{claim:PTR}]\label{claim:PTR_restated}
    \PTRclaimstatement
\end{claim}
\begin{proof}
    Let $\tau = \frac{2\log \frac{1-\delta}{\delta}}{\eps}+4$ and
    let $p(z)$ be the probability that $\PTR(z)$ passes, defined as follows:
    \begin{align}
        p(z) = \Pr\brackets{\PTR(z)=\pass}
            &= \begin{cases}
                1 & \text{if $z= 0$} \\
                0 & \text{if $z\ge \tau$} \\
                1 - e^{\frac{\eps}{2}(z-2)}\delta & \text{otherwise}
            \end{cases}.
    \end{align}
    What properties must $p(z)$ have?
    It satisfies the second and third conditions by construction, so we only have to argue privacy, and consider $p(z)$ and $p(z')$ for $\abs{z - z'}\le 2$.
    Without loss of generality, assume $z \le z'$.
    Since $p(\cdot)$ is monotone decreasing (and thus $p(z)\ge p(z')$), we must show that
    \begin{align}
         p(z) \le e^{\eps} p(z') + \delta.
    \end{align}
    Since $z' \le z + 2$, we know that $p(z') \ge p(z+2)$ so it suffices to show that
    \begin{align}
         p(z) \le e^{\eps} p(z+2) + \delta.
    \end{align}
    We will prove this via cases, dealing with the boundaries (near $0$ and $\tau$) first.
    (Note that $\tau$ in the lemma statement is slightly larger than $\frac{2\log \frac{1-\delta}{\delta}}{\eps}+4$ for simplicity; we prove the stronger statement.)

    \textbf{Case 1:} Suppose $z=0$.
        Then $p(z)=1$ and
            $p(z+2) = 1 - e^{\frac{\eps}{e}(2-2)}\delta = 1- \delta.$
        Thus $p(z) \le p(z+2)+\delta \le e^{\eps}p(z+2)+\delta$.

    \textbf{Case 2:}
        Suppose $z \ge \tau - 2$, so that $p(z+2)=0$.
        We calculate 
        \begin{align}
            p(z) \le p(\tau-2) &= 1 - e^{\frac{\eps}{2}(\tau - 2 - 2)} \delta \\
                &= 1 - e^{\frac{\eps}{2} \paren{\frac{2\log \frac{1-\delta}{\delta}}{\eps}}}\delta \\
                &= 1 - \frac{1-\delta}{\delta}\cdot\delta = \delta.
        \end{align}
        Thus $p(z)\le p(z+2)+\delta \le e^{\eps}p(z+2)+\delta$.

    \textbf{Case 3:}
        Now suppose $0<z<\tau-2$ and set $\delta'\defeq e^{\frac{\eps}{2}(z-2)}\delta$ for brevity.
        Since $0 < z$, we have $\delta' \ge  e^{-\eps}\delta$.
        We calculate
        \begin{align}
            p(z) = 1 - \delta' 
                \quad\text{and}\quad
            p(z+2) = 1 - e^{\frac{\eps}{2}(z+2-2)}\delta 
                = 1 - e^{\eps}\delta'.
        \end{align}
        We now $e^{\eps}p(z+2) - p(z)$ and show that it is positive:
        \begin{align}
            e^{\eps}p(z+2) - p(z) &= e^{\eps}\paren{1 - e^{\eps}\delta'} - \paren{1 - \delta'} \\
                &= e^{\eps} - e^{2\eps}\delta' - 1 + \delta' \\
                &= \paren{e^{\eps} - 1} - \delta'\paren{e^{2\eps} - 1}.
        \end{align}
        This is positive when $\frac{e^\eps - 1}{e^{2\eps}-1}\ge \delta' \ge e^{-\eps}\delta$.
        Since $\eps\le 1$, this is true. Observe:
        \begin{align}
            \frac{e^\eps - 1}{e^{2\eps}-1} &\ge \frac{\eps}{e^{2\eps}-1} \\
                &\ge \frac{\eps}{e^{2}-1}.
        \end{align}
        This is greater than $\delta'$ when $\delta \le \frac{e^\eps}{e^{2}-1}\cdot \eps$.
        Using $\frac{e^\eps}{e^{2}-1}\ge \frac{1}{10}$ finishes the proof, since we assumed $\delta\le \frac{\eps}{10}$.
\end{proof}

\end{document}